\documentclass{article}

 \usepackage[preprint]{neurips_2026}

\usepackage[utf8]{inputenc} 
\usepackage[T1]{fontenc}    
\usepackage{hyperref}       
\usepackage{url}            
\usepackage{booktabs}       
\usepackage{amsfonts}       
\usepackage{nicefrac}       
\usepackage{microtype}      
\usepackage{xcolor}         

\usepackage{amsmath,amssymb,amsthm,mathtools}

\usepackage{enumitem}
\usepackage{subcaption}
\usepackage{graphicx}
\usepackage{setspace}
\usepackage{comment}
\usepackage{xcolor}
\usepackage{multirow}

\newtheorem{theorem}{Theorem}
\newtheorem{lemma}{Lemma}
\newtheorem{definition}{Definition}
\newtheorem{corollary}[theorem]{Corollary}

\newcommand{\ind}{\mathbf{1}}
\newcommand{\R}{\mathbb{R}}
\newcommand{\E}{\mathbb{E}}
\newcommand{\Pp}{\mathbb{P}}

\title{Post-Selection Distributional Model Evaluation}

\author{
Amirmohammad Farzaneh \qquad Osvaldo Simeone \\
Institute for Intelligent Networked Systems (INSI) \\
Northeastern University London \\
London, UK \\
\texttt{\{a.farzaneh,o.simeone\}@nulondon.ac.uk}
}

\begin{document}

\maketitle

\begin{abstract}
Formal model evaluation methods typically certify that a model satisfies a prescribed target key performance indicator (KPI) level. However, in many applications, the relevant target KPI level may not be known a priori, and the user may instead wish to compare candidate models by analyzing the full trade-offs between performance and reliability achievable at test time by the models. This task, requiring the reliable estimate of the test-time KPI distributions, is made more complicated by the fact that the same data must often be used both to pre-select a subset of candidate models and to estimate their KPI distributions, causing a potential post-selection bias.
In this work, we introduce post-selection distributional model evaluation (PS-DME), a general framework for statistically valid distributional model assessment after arbitrary data-dependent model pre-selection. Building on e-values, PS-DME controls post-selection false coverage rate (FCR) for the distributional KPI estimates and we establish explicit conditions under which it is provably more sample efficient than a baseline method based on sample splitting.
Experiments on synthetic data, text-to-SQL decoding with large language models, and telecom network performance evaluation demonstrate that PS-DME enables reliable comparison of candidate configurations across a range of reliability levels, supporting the statistically reliable exploration of performance--reliability trade-offs.
\end{abstract}

\section{Introduction}
\label{sec:intro}

\begin{figure}[h!]
  \centering
\includegraphics[width = \textwidth]{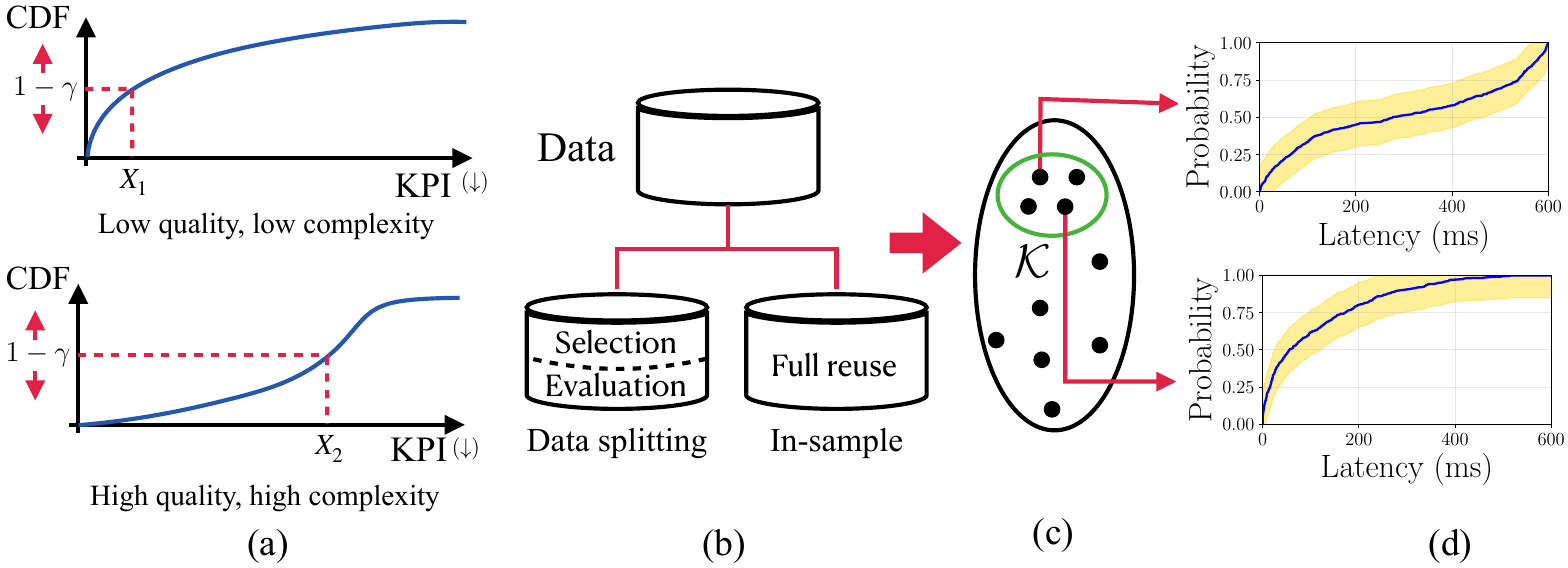}
  \caption{(a)~CDF of a negatively-oriented key performance indicator (KPI), such as the prefill latency for an LLM agent to be deployed on a device \citep{huang2026scale}, for two  model configurations.  For a test-time reliability level $1-\gamma$, knowing the CDF for each configuration allows the user to compare achievable performance across different configurations. Furthermore, it also supports the exploration of different reliability levels $1-\gamma$.
(b)~Two strategies for handling model pre-selection, namely data splitting and in-sample inference.
(c)-(d) Illustration of the proposed post-selection distributional model evaluation (PS-DME), which applies in-sample pre-selection: (c)~A user-defined selection strategy identifies a subset $\mathcal{K}$ of promising configurations among the $K$ candidates;
(d)~For all the configurations in set $\mathcal{K}$, PS-DME constructs confidence bands for the CDFs of the KPI of interest with guaranteed \textit{false coverage rate} (FCR), enabling reliable comparison of performance–reliability trade-offs across configurations.}
  \label{fig:motivating_example}
\end{figure}

Current formal model selection and evaluation strategies, grounded in \textit{multiple hypothesis testing} (MHT) \citep{angelopoulos2025learn, farzaneh2025ensuring}, control the probability that the risk measure of interest  exceeds a maximum tolerated target value. However, in practice, a suitable target risk level may not be known in advance, and the user may instead wish to explore different \textit{trade-offs} between test-time performance and reliability. As studied in this paper, this requires estimating the full test-time  distribution of the \emph{key performance indicator} (KPI) of interest for the candidate models.

To exemplify the problem of interest, consider a large language model (LLM) agent that can be deployed on a mobile device with different configurations $k = 1, \ldots, K$. As studied in \citep{huang2026scale}, these may correspond to models obtained by progressively pruning a base model. Accordingly, the available configurations range from high-quality and high-latency models to lower-quality and lower-latency models. Latency here refers to the  prefill delay, which dictates the time to first token. Existing formal model selection methods \citep{angelopoulos2025learn}, reviewed in \citep{farzaneh2025ensuring}, require the a priori specification of a latency target. However, as pointed out in \citep{huang2026scale}, the latency target may change as the use-cases evolve, and there is a need to understand the full trade-off between performance and reliability.

For example, consider two LLM configurations -- a lower-quality, simpler, model and a higher-quality, more complex, model. As illustrated in Fig. \ref{fig:motivating_example}(a), suppose that the user knew the \textit{cumulative distribution function} (CDF) of the latency for the two different configurations. Note that the higher quality model, due to its higher complexity, tends to have a larger latency. If the user tolerates a maximum test-time failure probability $\gamma$, it may prefer the higher-quality model, but only  as long as the benefits accrued from higher quality outweigh the larger latency ($X_2>X_1$ in the figure).  More generally, the user may also be interested to explore different values of the failure probability $\gamma$ in order to obtain more favorable operating points on the trade-off between performance (latency) and test-time reliability $1-\gamma$.


The model selection strategy outlined in the previous example is challenged by the fact that the CDFs of the KPI of interest, e.g., the latency, are not known and must be estimated from data. Furthermore, the configurations being evaluated are often pre-selected from a large pool of candidates to satisfy requirements or criteria such as cost or expected accuracy. Because this pre-selection must generally rely on the
same data later used for CDF estimation, naively
applying classical confidence bands in this setting yields overly
optimistic estimates (see, e.g.,~\citep{kuchibhotla2025post}).

We address these challenges by proposing \textit{post-selection distributional model evaluation} (PS-DME), a novel formal model evaluation protocol that supports the following steps:
\begin{enumerate}[label=\textbf{\arabic*)}, leftmargin=*]
    \item \textbf{Arbitrary model pre-selection:} As illustrated in Fig. \ref{fig:motivating_example}(c), the user can pre-select a subset $\mathcal{K}\subseteq\{1, \ldots, K\}$ of promising configurations using any arbitrary and possibly data-dependent procedure. As shown in Fig.~\ref{fig:motivating_example}(b), we study both \emph{data splitting}, where one portion of the data is used for pre-selection and the remaining portion is reserved for
distributional inference, and \emph{in-sample} inference, in which the same data serve both for model
screening and for uncertainty
quantification~\citep{chakraborty2026comparing}. 
    \item \textbf{Distributional model evaluation:} For each pre-selected configuration $k \in \mathcal{K}$, as illustrated in Fig.~\ref{fig:motivating_example}(d), PS-DME constructs confidence bands for the CDF of the KPI of interest. These bands quantify the uncertainty in the entire distribution of the performance metric, enabling the reliable estimation of quantities such as latency quantiles or outage probabilities (see Fig. \ref{fig:motivating_example}(a)).
    \item \textbf{Model selection based on performance-reliability trade-off with statistical guarantees:} The estimated CDF confidence bands enable the user to compare configurations by examining the trade-off between achievable performance and reliability levels. This allows the user to select a final configuration that strikes the desired balance between performance and reliability, while accounting for statistical uncertainty.
\end{enumerate}

The rest of this paper is organized as follows. In Sec.~\ref{sec:problem_setup}, we formalize the post-selection distributional model evaluation problem, and introduce the FCR-based validity criterion adopted throughout the paper. Sec.~\ref{sec:method} then presents the proposed PS-DME framework together with the sample-splitting baseline, establishes their statistical guarantees, and compares their statistical efficiency through explicit confidence band width comparisons. Sec.~\ref{sec:related_work} positions our contribution with respect to prior work. Sec.~\ref{sec:experiments} validates the proposed methodology on synthetic data, LLM decoding, and telecom performance evaluation. Sec.~\ref{sec:conclusion} concludes the paper with a discussion of the main findings and directions for future work, while the appendices collect proofs and extensions.

\section{Problem Setup}
\label{sec:problem_setup}

For a given system, we consider $K$ candidate configurations indexed by integer $k \in \{1,\dots,K\}$. These may correspond to different hyperparameter settings, decoding strategies, or precision levels. Henceforth, we will refer to each configuration $k$ as a \textit{model}.
For each model $k$, we observe $n_k$ i.i.d.\ realizations of a scalar, negatively-oriented, \textit{key performance indicator} (KPI),
\begin{equation}
\label{eq:dataset}
\mathcal{D}_k = \{X_{k,1},\dots,X_{k,n_k}\},\quad \text{with} \;X_{k,i}\ \stackrel{\mathrm{i.i.d.}}{\sim}\ P_k,
\end{equation}
where $P_k$ is an unknown distribution with CDF
$F_k(x) \ =\ \Pp_{k}\!\left( X \le x \right)$. The datasets $\mathcal{D}_k$ can be arbitrarily correlated. For instance, in the case with the same number of data points for all configurations, i.e., $n_k = n$, the $i$-th measures $\{X_{k,i}\}_{k = 1}^K$ may be a function of the same observations, with observations being i.i.d. across index $i = 1, \ldots, n$.

Let $\widehat F_k(\cdot)$ denote the empirical CDF formed from dataset $\mathcal{D}_k$ as
\begin{equation}
\label{eq:emp_cdf}
\widehat F_k(x) \ =\ \frac{1}{n_k}\sum_{i=1}^{n_k}\ind\{X_{k,i}\le x\}.
\end{equation}

We study an arbitrary two-step model evaluation pipeline consisting of the following phases:

\textbf{1) Data-based model pre-selection:} Using data $\mathcal{D} = \cup_{k = 1}^{K}\mathcal{D}_k$ for all $k$ configurations, an arbitrary (possibly unknown) data-dependent pre-selection rule produces a subset
\begin{equation}
\mathcal{K} \ =\ S\!\left(\mathcal{D}\right)\ \subseteq\ \{1,\dots,K\}
\end{equation}
of models.

\textbf{2) Reliable distributional model evaluation:} For each selected model $k\in \mathcal{K}$, we aim to report lower and upper bounds on the CDF $F_k(x)$, i.e.,
\begin{equation}
L_k(x)\ \le\ F_k(x)\ \le\ U_k(x),
\end{equation}
that contain the true CDF uniformly with a user-defined probability.

Formally, for each selected model $k\in \mathcal{K}$, define the miscoverage event
\begin{equation}
\label{eq:not_covered}
\mathcal{O}_k \ =\ \Bigl\{ \exists x\in\R:\ F_k(x)\notin [L_k(x),U_k(x)] \Bigr\},
\end{equation}
i.e., the event that there exists a KPI value $x$ for which the interval $[L_k(x),U_k(x)] $ does not contain its true CDF value $F_k(x)$ (see Appendix \ref{app:illustrations} for an illustration).
\textit{Post-selection statistical validity} is measured via the \textit{false coverage proportion} (FCP) \citep{xu2024post}. The FCP is defined as the fraction of  pre-selected models in set $\mathcal{K}$ for which the miscoverage event (\ref{eq:not_covered}) holds true, i.e.,
\begin{equation}
\label{eq:def:fcp}
\mathrm{FCP}\ =\ \frac{1}{\max\{|\mathcal{K}|,1\}}\sum_{k\in \mathcal{K}}\ind\{\mathcal{O}_k\}.
\end{equation}
The \textit{false coverage rate} (FCR) is the expectation of the FCP, i.e.,
\begin{equation}
\label{eq:def:fcr}
\mathrm{FCR}\ =\ \E\!\left[\mathrm{FCP}\right],
\end{equation}
where $\ind\{\cdot\}$ is the indicator function, and the expectation is taken with respect to the joint distribution of data 
$\mathcal{D}$. Note that the selected set $\mathcal{K}=S(\mathcal{D})$ is itself data-dependent, and hence its cardinality $|\mathcal{K}|$ is a random variable. Our goal is to guarantee the condition
\begin{equation}
\label{eq:fcr_guarantee}
\mathrm{FCR}\le \delta
\end{equation}
for a user-defined target probability $\delta\in(0,1)$.

\section{Post-Selection Distributional Model Evaluation}
\label{sec:method}

In order to ensure the condition (\ref{eq:fcr_guarantee}), a conventional approach would adopt a sample splitting strategy, followed by the evaluation of standard uniform confidence intervals for CDFs \citep{dvoretzky1956asymptotic, massart1990tight}. After briefly reviewing this baseline method in Sec. \ref{sec:sample_split}, we introduce PS-DME in Sec. \ref{sec:in-sample}. As introduced in Sec. \ref{sec:intro}, PS-DME leverages the entire dataset $\mathcal{D}$ for both model selection and distributional model evaluation tasks (see Fig. \ref{fig:motivating_example}).

\subsection{Sample-Splitting Distributional Model Evaluation}
\label{sec:sample_split}

For each configuration $k \in \{1,\dots,K\}$, suppose the available data
$\mathcal{D}_k = \{X_{k,1},\dots,X_{k,n_k}\}$ in (\ref{eq:dataset}) are partitioned into two
disjoint subsets
$\mathcal{D}_k^{\mathrm{sel}}$
and
$\mathcal{D}_k^{\mathrm{eval}}$
of sizes $n_k^{\mathrm{sel}}$ and $n_k^{\mathrm{eval}}$ respectively. In the baseline methodology reviewed here, referred to as \textit{sample-splitting post-selection distributional model evaluation} (SS-DME),
the pre-selection subset $\mathcal{D}_k^{\mathrm{sel}}$ is used exclusively for model selection,
while the model evaluation subset $\mathcal{D}_k^{\mathrm{eval}}$ is reserved for distributional inference.

\textbf{1) Model selection:} The selection rule
$\mathcal{K}
=
S\!\left(
\mathcal{D}^{\mathrm{sel}}
\right)
\subseteq
\{1,\dots,K\}$
leverages the dataset
$\mathcal{D}^{\mathrm{sel}}
=
\cup_{k=1}^K \mathcal{D}_k^{\mathrm{sel}}$ using any arbitrary criterion. For instance, the function $S(\cdot)$ may select the $\bar{K} \leq K$ models with the smallest empirical KPIs
$\hat{X}_k^\text{sel}= \sum_{i = 1}^{n_k^\text{sel}} X_{k,i}/n_k^\text{sel}$.

\textbf{2) Distributional inference:} For each selected configuration $k \in \mathcal{K}$, SS-DME constructs
a uniform confidence band for the true CDF $F_k (\cdot)$
using only the inference subset
$\mathcal D_k^{\mathrm{eval}}$. This is elaborated on in the rest of this subsection.

Let $\widehat F_k^{\mathrm{eval}}(\cdot)$ denote the empirical CDF
constructed from data $\mathcal D_k^{\mathrm{eval}}$ using \eqref{eq:emp_cdf}.
By the Dvoretzky--Kiefer--Wolfowitz (DKW) inequality \citep{dvoretzky1956asymptotic},
for any $\varepsilon > 0$ we have
\begin{equation}
\label{eq:dkw_split}
\mathbb P\!\left(
\sup_{x \in \mathbb R}
\big|
\widehat F_k^{\mathrm{eval}}(x) - F_k(x)
\big|
>
\varepsilon
\right)
\le
2 \exp\!\left(-2 n_k^{\mathrm{eval}} \varepsilon^2 \right).
\end{equation}
Using this inequality, SS-DME constructs the confidence bounds as
\begin{subequations}\label{eq:split_bounds}
\begin{align}
L_k(x)
&=
\max\!\{
0,\,
\widehat F_k^{\mathrm{eval}}(x)
-
\varepsilon_k
\},
\label{eq:split_lower}
\\
\text{and}\quad
U_k(x)
&=
\min\!\{
1,\,
\widehat F_k^{\mathrm{eval}}(x)
+
\varepsilon_k
\},
\label{eq:split_higher}
\end{align}
\end{subequations}
where 
\begin{equation}
\label{eq:width_split_new}
\varepsilon_k
=
\sqrt{
\frac{1}{2 n_k^{\mathrm{eval}}}
\log\!\left(
\frac{2}{\delta}
\right)
}.
\end{equation}
SS-DME provides the following guarantee.

\begin{lemma}
\label{theorem:validity_split}
For any pre-selection rule $S(\cdot)$, the confidence bands $\mathcal{C}_k = [L_k(x), U_k(x)]$ produced by SS-DME, where $L_k(x)$ and $U_k(x)$ are defined by
\eqref{eq:split_lower} and \eqref{eq:split_higher}, respectively,
satisfy the requirement $\mathrm{FCR} \le \delta$.
\end{lemma}

A proof of Lemma \ref{theorem:validity_split} can be found in Appendix \ref{app:proof_lemma_1}.

\subsection{In-Sample Distributional Model Evaluation}
\label{sec:in-sample}

The SS-DME strategy described in Sec.~\ref{sec:sample_split}
ensures validity by enforcing independence between
model selection and distributional inference.
However, this separation reduces the effective sample size
available for constructing confidence bands, potentially leading to
occasionally wide intervals. In this subsection, we introduce PS-DME, which enables full-data reuse while maintaining
the requirement (\ref{eq:fcr_guarantee}) on the FCR. 
To achieve this, PS-DME adopts the framework of e-values \citep{ramdas2025hypothesis} to construct confidence intervals with post-selection validity \citep{xu2024post}.

E-values are statistics used to test null hypotheses that have additional properties of robustness as compared to p-values, including, notably, the support of the post-hoc selection of the significance level \citep{vovk2021values, koning2025posthocalphahypothesistesting}. An e-value can be constructed from a p-value via the application of an e-calibrator function, which is defined as follows.

\begin{definition}[e-calibrator]
A function $f:[0,1]\to[0,\infty]$ is called an e-calibrator if it satisfies the following properties:
\textit{(i)} $f$ is nonincreasing and upper semicontinuous; and
\textit{(ii)} $\displaystyle \int_0^1 f(u)\,du \le 1$.

\end{definition}


For example, a commonly used family of e-calibrators is given by the functions
\begin{equation}
\label{eq:power_e-calibrator}
f_\tau(p) = (1-\tau) \, p^{-\tau}, \qquad \tau \in (0,1).
\end{equation}
An optimal e-calibrator would yield the largest e-value across all possible p-values $p\in[0,1]$. However, within the family (\ref{eq:power_e-calibrator}), no single value of parameter $\tau$ yields the largest e-value $f_\tau(p)$ for all possible values of $p\in [0,1]$ \citep{vovk2021values}. Therefore, the optimal choice of $\tau$ is non-trivial, and a common selection is $\tau=1/2$ \citep{yanchenko2025hypothesis}.

For any e-calibrator $f(\cdot)$, define its generalized inverse as
$f^{-1}(t)
=
\sup\{u\in[0,1]: f(u)\ge t\}$.

For each configuration $k \in \{1,\dots,K\}$,
let $\widehat F_k(\cdot)$ denote the empirical CDF formed by \eqref{eq:emp_cdf}
using the entire dataset $\mathcal D_k$. For an arbitrary e-calibrator $f(\cdot)$, PS-DME produces the bounds
\begin{subequations}\label{eq:in_sample_bounds}
\begin{align}
L_k(x)
&=
\max\!\{
0,\,
\widehat F_k(x)
-
\varepsilon_k
\},
\label{eq:e_lower}
\\
\text{and}\quad
U_k(x)
&=
\min\!\{
1,\,
\widehat F_k(x)
+
\varepsilon_k
\},
\label{eq:e-upper}
\end{align}
\end{subequations}
with band width
\begin{equation}
\varepsilon_k
=
\sqrt{
\frac{1}{2n_k}
\log\!\left(
\frac{2}{
f^{-1}\!\left(\frac{K}{\delta |\mathcal{K}|}\right)
}
\right)
}.
\label{eq:width_general_new}
\end{equation}

Using the properties of e-values \citep{vovk2021values}, PS-DME can be shown to have the following validity property, which is proved in Appendix \ref{app:theorem_proof}.

\begin{theorem}
\label{theorem:validity_full}
For any selection rule $S(\cdot)$ and for any e-calibrator $f(\cdot)$,
the confidence bands in (\ref{eq:e_lower}) and (\ref{eq:e-upper}) produced by PS-DME
satisfy the requirement (\ref{eq:fcr_guarantee}), i.e., $\mathrm{FCR} \le \delta$.
\end{theorem}

An alternative implementation of PS-DME that leverages Berk-Jones (BJ) confidence bounds \citep{berk1979goodness} instead of the DKW inequality (\ref{eq:dkw_split}) is presented in Appendix \ref{app:generalization}. BJ confidence bounds are implicitly defined based on likelihood ratios, yielding input-dependent confidence band widths. This contrasts with the sup-norm form of the DKW statistic (\ref{eq:dkw_split}), which yields efficient, symmetric, and input-independent confidence widths (\ref{eq:width_general_new}).

\subsection{Comparing SS-DME and PS-DME}
\label{sec:compare_methods}

Thanks to its use of the full dataset for distributional inference, PS-DME can be more efficient than SS-DME. In fact, a direct comparison of the band widths \eqref{eq:width_general_new} and \eqref{eq:width_split_new} yields an explicit
condition under which PS-DME yields a narrower band than SS-DME as follows.

\begin{lemma}
\label{lemma:compare_ps_dme_split}
For a pre-selection set size $|\mathcal{K}|$, the band width of PS-DME (\ref{eq:width_general_new}) is smaller than that of SS-DME (\ref{eq:width_split_new})
if and only if the inequality
\begin{equation}
\label{eq:compare_width_general_vs_split_rho}
\frac{n_k^\mathrm{eval}}{n_k}
<
\frac{
\log\!\left(\frac{2}{\delta}\right)
}{
\log\!\left(\frac{2}{f^{-1}\!\left(\frac{K}{\delta|\mathcal K|}\right)}\right)
}
\end{equation}
holds.
\end{lemma}

The condition (\ref{eq:compare_width_general_vs_split_rho}) highlights the fundamental trade-off between the two methods.
PS-DME benefits from using the full sample size $n_k$ for both pre-selection and inference, but it pays an additional
post-selection correction factor captured by the term $f^{-1}\!\left(K/{\delta|\mathcal K|}\right)$.
In contrast, SS-DME avoids this correction by enforcing separation between
pre-selection and inference, but it uses only the reduced sample size
$n_k^{\mathrm{eval}}$.

Using the e-calibrator family (\ref{eq:power_e-calibrator}), the result in Lemma \ref{lemma:compare_ps_dme_split} can be specialized as follows (see Appendix \ref{app:cor_proof} for a proof).

\begin{corollary}
\label{cor:vs_condition}
Assume that the e-calibrator belongs to the power family
\eqref{eq:power_e-calibrator}.
A necessary condition for the PS-DME width (\ref{eq:width_general_new}) to be smaller than the
SS-DME width (\ref{eq:width_split_new}) is
\begin{equation}
\label{eq:vs_condition}
\frac{n_k^\mathrm{eval}}{n_k}
<
\frac{
\log\!\left(\frac{2}{\delta}\right)
}{
\log 2
-
W\!\left(-\frac{\delta|\mathcal K|}{eK}\right)
},
\end{equation}
where $W(\cdot)$ denotes the lower real branch of the Lambert $W$ function,
i.e., the inverse of the function $f(w) =  w e^{w}$ on the interval
$[-1/e,0)$. Furthermore, when the candidate set size $|\mathcal{K}|$ is fixed and does not depend on the data $\mathcal{D}_k$, if condition (\ref{eq:vs_condition}) holds, there exists a value of the constant $\tau$ for which the width (\ref{eq:width_general_new}) of PS-DME is smaller than the width of SS-DME (\ref{eq:width_split_new}), namely
\begin{equation}
\label{eq:optimal_calibrator}
    \tau = 1 + \frac{1}{W\!\left(-\frac{\delta|\mathcal{K}|}{eK}\right)}.
\end{equation}
\end{corollary}
An illustration of this result can be found in Appendix \ref{app:illustrations}.

\section{Related Work}
\label{sec:related_work}

Referring to Appendix \ref{app:related_work} for further details, we highlight here connections to the literature on post-selection inference.

In the literature on inference after data-dependent selection \citep{lee2016exact, taylor2018post}, the FCR criterion was introduced in \citep{benjamini2005false} as an analog of the false discovery rate for confidence intervals. More recently, reference \citep{xu2024post} developed general procedures for controlling FCR under arbitrary data-dependent selection. This methodology leverages e-values for their in-sample post-selection validity \citep{grunwald2024beyond, vovk2021values}. Related work has clarified the connection between e-values, post-hoc p-values, and data-dependent testing procedures \citep{koning2025posthocalphahypothesistesting, chugg2026post}.

In the context of language models, prompt risk control \citep{zollo2024prompt} develops upper confidence bounds for a family of scalar risk functionals of the loss distribution, including value-at-risk and conditional value-at-risk. Furthermore, \citep{feng2026noisy} studies Type-I error control when the KPI is itself produced by a noisy LLM judge. This is complementary to our framework, which can be instantiated with any judge-based KPI (see Appendix~\ref{app:judge}).

In concurrent work, leveraging the framework of e-values, reference \citep{koobs2026equivalencetestingdatadependentposthoc} developed a mechanism for obtaining confidence intervals based on data-dependent confidence levels. Framing the contribution in the context of the present work, the method in \citep{koobs2026equivalencetestingdatadependentposthoc} applies to a single model and to a fixed statistic, e.g., the average risk. Therefore, it cannot be leveraged for model selection, as well as for exploring the test-time trade-offs between performance and reliability.

\section{Experiments}
\label{sec:experiments}

\subsection{Benchmarks and Performance Measures}

In this section, we compare three distributional model evaluation strategies by leveraging different metrics capturing reliability and efficiency for both performance evaluation and final selection.

\subsubsection{Distributional Model Evaluation Strategies}
\label{sec:benchmarks}

We compare the following strategies:

\textbf{1) Sample-splitting distributional model evaluation (SS-DME):}
As described in
Sec.~\ref{sec:sample_split}, SS-DME partitions the dataset for each configuration into
independent selection and inference subsets.
Model selection is carried out using the selection split,
and CDF bands are constructed on the inference split via the
DKW-based confidence intervals (\ref{eq:split_lower}) and (\ref{eq:split_higher}).

\textbf{2) Naive in-sample distributional model evaluation (N-PS-DME):}
This benchmark reuses the full dataset $\mathcal{D}$
for both model selection and distributional inference,
applying the same construction (\ref{eq:split_lower})-(\ref{eq:split_higher}) as in SS-DME with $n_k$ in place of
$n_k^{\mathrm{eval}}$. Due to selection bias, N-PS-DME cannot generally guarantee the FCR requirement (\ref{eq:fcr_guarantee}).

\textbf{3) In-sample distributional model evaluation (PS-DME):} The proposed PS-DME, detailed in
Sec.~\ref{sec:in-sample}, reuses the full dataset for both selection and inference, leveraging an e-value-based methodology to construct valid post-hoc CDF bounds (\ref{eq:e_lower}) and (\ref{eq:e-upper}) that satisfy the FCR condition (\ref{eq:fcr_guarantee}). Unless otherwise stated, we use the power e-calibrator (\ref{eq:power_e-calibrator}) with the optimal parameter $\tau$ calculated as discussed in Sec. \ref{sec:compare_methods}.

\subsubsection{Performance Measures}

We adopt the following metrics to assess both reliability and efficiency of different strategies.

\textbf{1) False coverage proportion (FCP):}
For each simulated trial,
we compute the realized FCP as defined in
\eqref{eq:def:fcp},
namely the fraction of selected configurations $k\in\mathcal{K}$
for which the miscoverage event
$\mathcal O_k$ in \eqref{eq:not_covered} occurs.
Averaging across repetitions yields an empirical estimate
of the FCR (\ref{eq:def:fcr}).

\textbf{2) Confidence band width:}
We quantify statistical efficiency via the average
uniform band radius.
For SS-DME and N-PS-DME this corresponds to
$\varepsilon_k$ in
\eqref{eq:width_split_new},
and for PS-DME to
$\varepsilon_k$ in
\eqref{eq:width_general_new}.
We report averages across selected configurations
and experimental repetitions.

\textbf{3) Post-selection guaranteed KPI:}
With reference to the example in Fig. \ref{fig:motivating_example}(a), suppose that the system under study has a maximum allocated failure probability equal to $\gamma$. Using the lower bound $L_k(x)$, each configuration $k$ can guarantee a KPI no worse than
\begin{equation}
\label{eq:best_guaranteed_k}
X_k^* = \inf \{x: L_k(x)\ge 1-\gamma\},
\end{equation}
while satisfying an error rate (FCR) no larger than $\delta$. We provide an illustration of this definition in Fig. \ref{fig:best_quantile_bounds_two_n}, which is further discussed in the next section. Upon optimization over the configuration $k\in \mathcal{K}$, the \textit{post-selection best guaranteed KPI} is thus given by $X^* = \min_{k\in \mathcal{K}} X^*_k$.

\subsection{Synthetic Data}
\label{sec:synthetic}

We begin by validating the proposed PS-DME in a fully
controlled synthetic setting, where the true data-generating distribution is
known exactly. This allows us to directly assess both coverage properties and
the tightness of the resulting CDF bands.

\textbf{Setup:} We consider a standard linear-Gaussian model with
covariates drawn as
$U \sim \mathcal{N}(0, I_{10})$,
and responses generated according to the linear model
$V = U^\top \beta + \varepsilon$, with noise $\varepsilon \sim \mathcal{N}(0, 1)$
and coefficient vector $\beta \sim \mathcal{N}(0,I_{10})$
whose entries are held
fixed for all data points.

We construct $K = 2000$ candidate predictors $\widehat{g}_k(\cdot)$ by training ridge regression models over
a logarithmically spaced grid of regularization parameters
$\lambda_1,\dots,\lambda_K$, where the grid spans the interval
$[10^{-4}, 10^{2}]$.
Each model is trained on an independent training dataset of size 600.
For each model $k$, the KPI is defined as the squared prediction error
$X_{k,i} = (V_i - \widehat{g}_k(U_i))^2$,
evaluated on the same independent calibration dataset $\{(U_i, V_i)\}_{i = 1}^n$ of size $n$.
Thus, for each model $k$, we have the KPI samples $\mathcal{D}_k = \{X_{k,i}\}_{i=1}^{n}$ that are i.i.d.\ draws from
an unknown distribution $P_k$, and samples $\{X_{k,i}\}_{k = 1}^K$ are correlated across different hyperparameters. Model pre-selection chooses the $M = 1000$ configurations with the smallest
empirical mean KPI. Throughout all experiments, we target an FCR level
$\delta = 0.1$ in the requirement (\ref{eq:fcr_guarantee}).

\textbf{Results:} To start, we illustrate the per-configuration CDF estimates produced by different schemes for
a given candidate hyperparameter, selected here as $\lambda_k = 0.0231$. Fig. \ref{fig:cdf_band_zoom} plots the corresponding true risk CDF $F_k(x)$, the empirical CDF $\widehat{F}_k(x)$, and the corresponding post-selection confidence band $[L_k(x), U_k(x)]$ provided by SS-DME, N-PS-DME, and PS-DME using a calibration dataset of size $n = 20$. For SS-DME, we use 50\% of the available calibration dataset for pre-selection, and the remainder 50\% for inference. In this example, while SS-DME and PS-DME provide miscoverage control, N-PS-DME exhibits a miscoverage event. This illustrates the general theoretical result that SS-DME and PS-DME offer the FCR requirement (\ref{eq:fcr_guarantee}), while N-PS-DME does not provide error control guarantees due to selection bias.

Fig. \ref{fig:cdf_band_zoom} also confirms the comparison between the widths of SS-DME and PS-DME formalized by Lemma \ref{lemma:compare_ps_dme_split}. For the given split, the condition (\ref{eq:compare_width_general_vs_split_rho}) is satisfied, and accordingly, the band width achieved by PS-DME is markedly narrower than that of SS-DME. As detailed in Sec. \ref{sec:in-sample}, this can be attributed to the fact that PS-DME uses the entire calibration dataset for forming the bands, while SS-DME is forced to use only a subset of the dataset for inference. As highlighted in the figure, this improved efficiency allows PS-DME to select a KPI $X^*_k$ in (\ref{eq:best_guaranteed_k}) that is smaller, and thus preferable, than SS-DME.

\begin{figure}
    \centering
    \includegraphics[width=\linewidth]{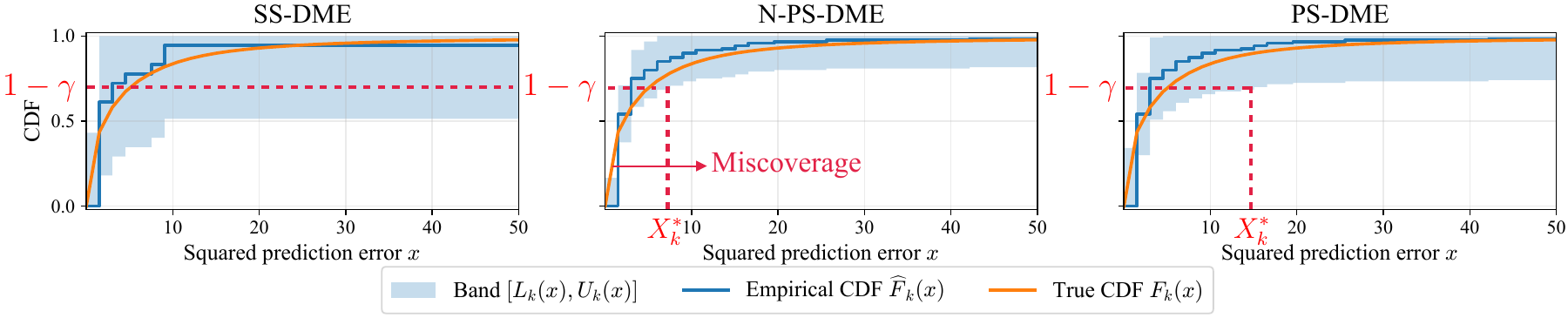}
\caption{
For a candidate hyperparameter $\lambda_k = 0.0231$, we show the true CDF $F_k(x)$, the empirical CDF $\widehat{F}_k(x)$, and the post-selection confidence band $[L_k(x), U_k(x)]$ for the data splitting, naive in-sample, and post-hoc in-sample benchmarks.
}
    \label{fig:cdf_band_zoom}
\end{figure}

\begin{figure}
    \centering
    \begin{subfigure}[t]{0.48\linewidth}
        \centering
        \includegraphics[width=\linewidth]{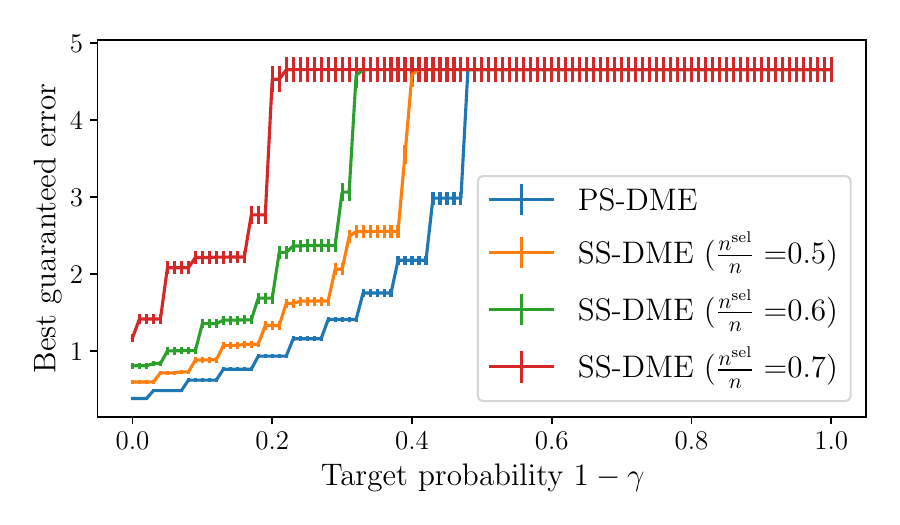}
        \caption{Calibration size $n=20$.}
    \end{subfigure}
    \hfill
    \begin{subfigure}[t]{0.48\linewidth}
        \centering
        \includegraphics[width=\linewidth]{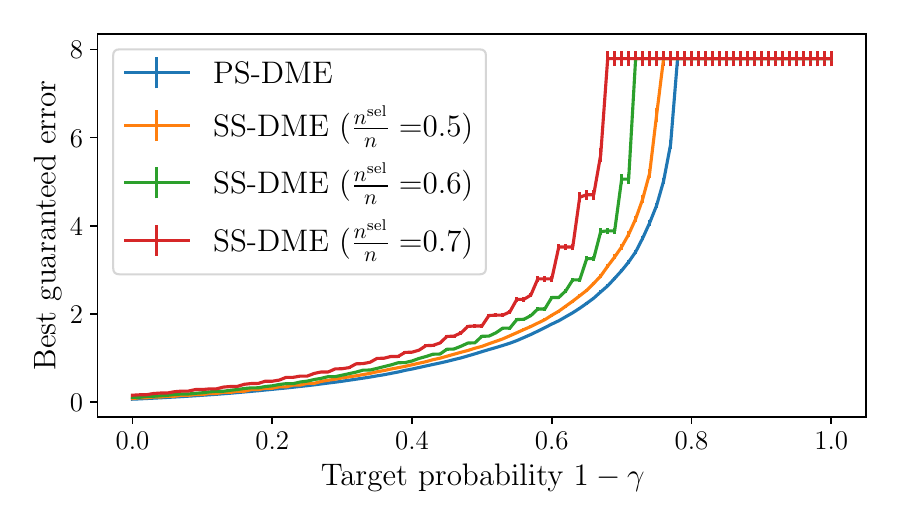}
        \caption{Calibration size $n=100$.}
    \end{subfigure}
    \caption{Best guaranteed KPI for the synthetic data experiment as a function of the target probability $1-\gamma$ for SS-DME and PS-DME. For SS-DME, a fraction $n^\text{sel}/n\in\{0.5, 0.6, 0.7\}$ of the available data is used for model pre-selection, while the remaining samples are used to construct the CDF confidence bands. (a) Results with 20 calibration samples. (b) Results with 100 calibration samples. The vertical error bars show the standard error of the mean across 100 random splits of the dataset.}
    \label{fig:best_quantile_bounds_two_n}
\end{figure}

We now turn to further analyzing the post-selection guaranteed KPI. Fig.~\ref{fig:best_quantile_bounds_two_n} reports the best post-selection guaranteed KPI as a function of the target probability $\gamma$ for SS-DME and for PS-DME. N-PS-DME does not guarantee FCR control, and thus is not included in this analysis. PS-DME is seen to achieve uniformly lower KPI (quadratic error) across
all values of target probability $1-\gamma$. Highest benefits are observed in the regimes of low test-time failure rates $\gamma$ and when calibration data are scarce, i.e., for $n = 20$.
In fact, in a low-sample regime, SS-DME suffers from a severe reduction
in effective inference sample size, since only a fraction of the already
limited data is available for band construction. Empirical FCR results are reported in Appendix~\ref{app:fcr_table}.

\subsection{Evaluating and Selecting LLM Decoding Configurations}
\label{sec:exp_spider_llm}

We next instantiate the proposed PS-DME in a
text-to-SQL task performed by an LLM, where the KPI is a score accounting for the quality of the generated SQL query.

\textbf{Task and dataset:} We consider the text-to-SQL Spider benchmark \citep{yu2018spider}, in which the LLM receives as input (\textit{i}) a natural-language question, (\textit{ii}) a database identifier specifying which database the question refers to, and (\textit{iii}) the corresponding database schema, including table names, columns, and their relationships. The goal of the model is to generate an SQL query that, when executed on the specified database, correctly answers the natural-language question. Performance is evaluated by comparing the execution result of the generated query with that of the gold (ground-truth) SQL query provided in the dataset. In our experiments, we restrict attention to a fixed subset of $n=200$ (randomly selected) validation examples from Spider \citep{yu2018spider}.

We adopt the
\texttt{Qwen/Qwen2.5-7B-Instruct} LLM \citep{yang2024qwen2technicalreport} with nucleus (top-$p$) sampling
\citep{holtzman2019curious}, with parameters given by
$\texttt{top\_p} \in \{0.90, 0.95\}$ and
temperature $\tau \in \{0.3, 0.7, 1.0\}$. The maximum output length is $160$ tokens.

The hyperparameters correspond to decoding strategies that trade off
computational budget against solution quality.
In total, we consider $K=43$ candidate configurations,
organized into three families:

\textbf{(i) Single decode:}
We include one greedy (deterministic) decoding together with the
six stochastic single-pass decoding strategies obtained by varying $\texttt{top\_p} \in \{0.90, 0.95\}$ and temperature $\tau \in \{0.3, 0.7, 1.0\}$.

\textbf{(ii) Best-of-$n$ (BoN) decoding:}
Best-of-$n$ sampling \citep{brown2020language,wang2022self}
generates $n\in \{2,4,6\}$ independent SQL candidates for the same input,
and selects the one achieving the highest reward.
For each value of parameter $n$, we vary
temperature $\tau \in \{0.3, 0.7, 1.0\}$
and $\texttt{top\_p} \in \{0.90, 0.95\}$,
yielding 18 BoN configurations.

\textbf{(iii) Iterative draft–refine (IAD) decoding:}
Iterative self-refinement methods \citep{madaan2023self}
improve outputs by repeatedly conditioning on previous drafts.
Starting from an initial draft SQL generated from the prompt,
the model produces revised SQL queries conditioned on the previous draft.
This process is repeated for up to $t\in \{2,4,6\}$ iterations,
after which the final SQL query is returned.
For each value of $t$, we vary
temperature $\tau \in \{0.3, 0.7, 1.0\}$
and $\texttt{top\_p} \in \{0.90, 0.95\}$,
yielding 18 IAD configurations.

To evaluate the quality of a generated SQL query, we construct a performance metric that captures both structural similarity at the token level and semantic correctness relative to the reference query. The token-level KPI measures the overlap between the SQL tokens appearing in the predicted and reference queries using the F1 score, i.e., the harmonic mean of precision and recall. The semantic correctness measure determines whether the generated query retrieves the correct information from the database leveraging the F1 score \citep{yu2018spider}. The final KPI is given by $1 - wF1_\text{token}
-
(1 - w)F1_\text{exec}$, where $F1_\text{token}$ and $F1_\text{exec}$ are the discussed F1 scores. We set $w = 0.3$, assigning greater importance to execution-level correctness, since the primary objective of a text-to-SQL system is to retrieve the correct database result. Model pre-selection chooses the $M=30$ configurations with the smallest empirical mean KPI.

\textbf{Results:} Fig.~\ref{fig:spider_quantile_compare} compares the post-selection guaranteed KPI achieved by SS-DME, which uses a fraction $n^\text{sel}/n\in\{0.6, 0.7, 0.8\}$ of the dataset for pre-selection, with that obtained by the proposed PS-DME. As shown in the figure, PS-DME consistently guarantees lower KPI values, corresponding to higher performance for the selected decoding strategy.

\begin{figure}[t]
    \centering
    \includegraphics[width = 0.65\linewidth]{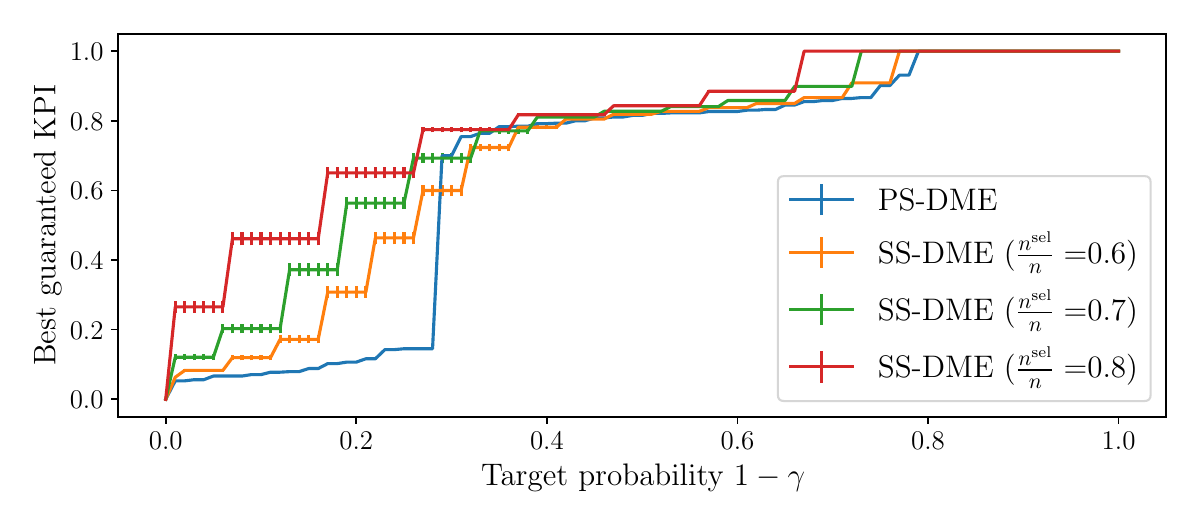}
    \caption{Best guaranteed KPI (loss) for the Spider text-to-SQL experiment as a function of the target probability $1-\gamma$ for SS-DME and PS-DME. For SS-DME, a fraction $n^\text{sel}/n\in\{0.6, 0.7, 0.8\}$ of the available data is used for decoding strategy pre-selection, while the remaining samples are used to construct the CDF confidence bands. Results are shown for $n=50$ calibration samples. The vertical error bars show the standard error of the mean across 100 random splits of the dataset.}
    \label{fig:spider_quantile_compare}
\end{figure}

{\section{Limitations}
\label{sec:limitations}

The validity guarantees of PS-DME in Theorem \ref{theorem:validity_full} rests on three assumptions that we restate below.

\paragraph{Independent samples:} For each configuration $k$, the observations in dataset $\mathcal{D}_k$ are assumed to be i.i.d., and this PS-DME is not directly applicable to time-series data.

\paragraph{Scalar KPI:} PS-DME evaluates one scalar metric at a time. Extending the framework to jointly evaluate several metrics simultaneously is left to future work.

\paragraph{Choice of e-calibrator:} As stated in Theorem~\ref{theorem:validity_full}, the FCR guarantee holds for any valid e-calibrator, but the width of the resulting confidence bands does depend on the parameter $\tau$. Corollary~\ref{cor:vs_condition} gives the optimal $\tau$ when the number of selected configurations $|\mathcal{K}|$ is known in advance, and Appendix~\ref{app:tau_sensitivity} provides a sensitivity analysis on the choice of $\tau$.
}

\section{Conclusion}
\label{sec:conclusion}

We introduced PS-DME, a framework for statistically valid distributional evaluation of pre-selected models under full data reuse. By leveraging e-values, the proposed method constructs post-selection confidence bands for entire CDFs with guaranteed false coverage rate control, addressing the bias induced by data-dependent model selection.

Unlike conventional target-based approaches, PS-DME enables uncertainty-aware assessment of the full KPI distribution, supporting reliable exploration of performance--reliability trade-offs across operating points such as latency quantiles and outage levels. Compared to sample splitting, PS-DME can yield narrower confidence bands when the efficiency gains from full data reuse outweigh the cost of post-selection correction, as formalized through explicit width comparisons.

Experiments on synthetic data, LLM decoding, and telecom evaluation demonstrate that PS-DME achieves valid post-selection inference while often providing tighter bands and stronger guarantees than sample-splitting baselines.

Future work includes extending the framework to more general and multivariate performance measures, developing sharper confidence band constructions, and integrating distributional guarantees into downstream decision-making for reliable model deployment.

\section*{Acknowledgments}
This work was supported by the European Research Council (ERC) under the European Union’s Horizon Europe Programme (grant agreement No. 101198347). The work of O. Simeone was also supported by an EPSRC Open Fellowship (EP/W024101/1) and by the EPSRC project (EP/X011852/1).


\newpage

\appendix

\section{Additional Illustrations}
\label{app:illustrations}

In this section, we provide additional illustrations for concepts introduced in the main text.

Fig. \ref{fig:miscoverage_example} illustrates the definition of a miscoverage event $\mathcal{O}_k$, as defined in \eqref{eq:not_covered}.

\begin{figure}[h!]
    \centering
    \begin{subfigure}[t]{0.48\linewidth}
        \centering
        \includegraphics[width=\linewidth]{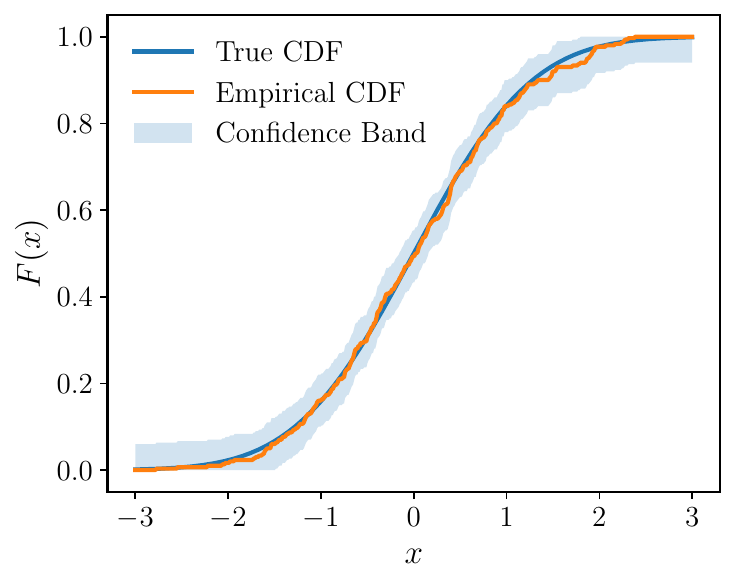}
        \caption{}
    \end{subfigure}
    \hfill
    \begin{subfigure}[t]{0.48\linewidth}
        \centering
        \includegraphics[width=\linewidth]{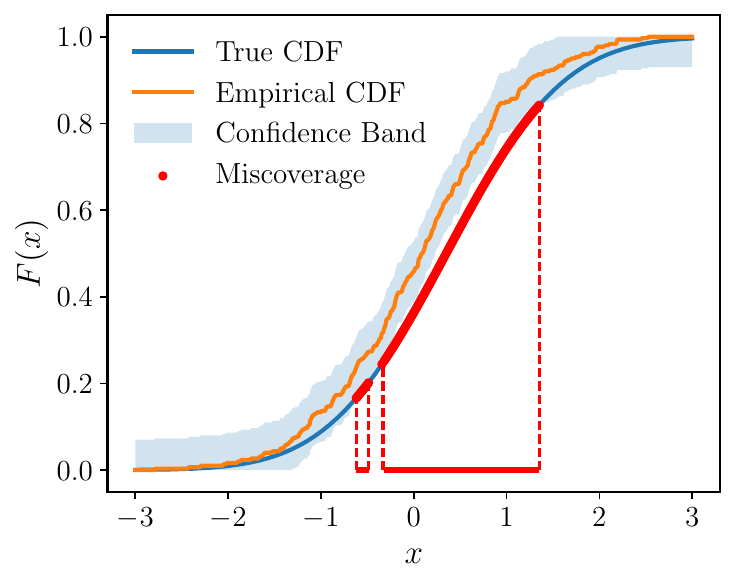}
        \caption{}
    \end{subfigure}
\caption{
Illustration of the miscoverage event $\mathcal{O}_k$ defined in \eqref{eq:not_covered}.
(a) A scenario in which the miscoverage event $\mathcal{O}_k$ in (\ref{eq:not_covered}) does not occur, as the true CDF lies entirely within the reported confidence band.
(b) A scenario in which the miscoverage event $\mathcal{O}_k$ in (\ref{eq:not_covered}) occurs, since the true CDF exits the band over the region of input $x$ marked in red.
}
    \label{fig:miscoverage_example}
\end{figure}

Fig.~\ref{fig:width_analysis} illustrates the band width (\ref{eq:width_general_new}) achieved by PS-DME when using the e-calibrators in (\ref{eq:power_e-calibrator}) with selected values of parameters $\tau$, as well as with the optimal calibrator parameter $\tau$ in \eqref{eq:optimal_calibrator}, together with the band width (\ref{eq:width_split_new}) obtained by SS-DME, as a function of the dataset split ratio $n_k^\text{eval}/n_k$ for fixed values $\delta = 0.1$ and $\delta|\mathcal{K}|/K = 0.01$. The intersection of the curves corresponding to PS-DME and SS-DME occurs at the threshold (\ref{eq:compare_width_general_vs_split_rho}) identified in Lemma \ref{lemma:compare_ps_dme_split} beyond which SS-DME yields a smaller width. It is observed that the choice of $\tau$ in (\ref{eq:optimal_calibrator}) yields the larger range of values of the ratio $n_k^\mathrm{eval}/n_k$ over which PS-DME is advantageous over SS-DME.

\begin{figure}[h!]
    \centering
    \includegraphics[width=0.7\linewidth]{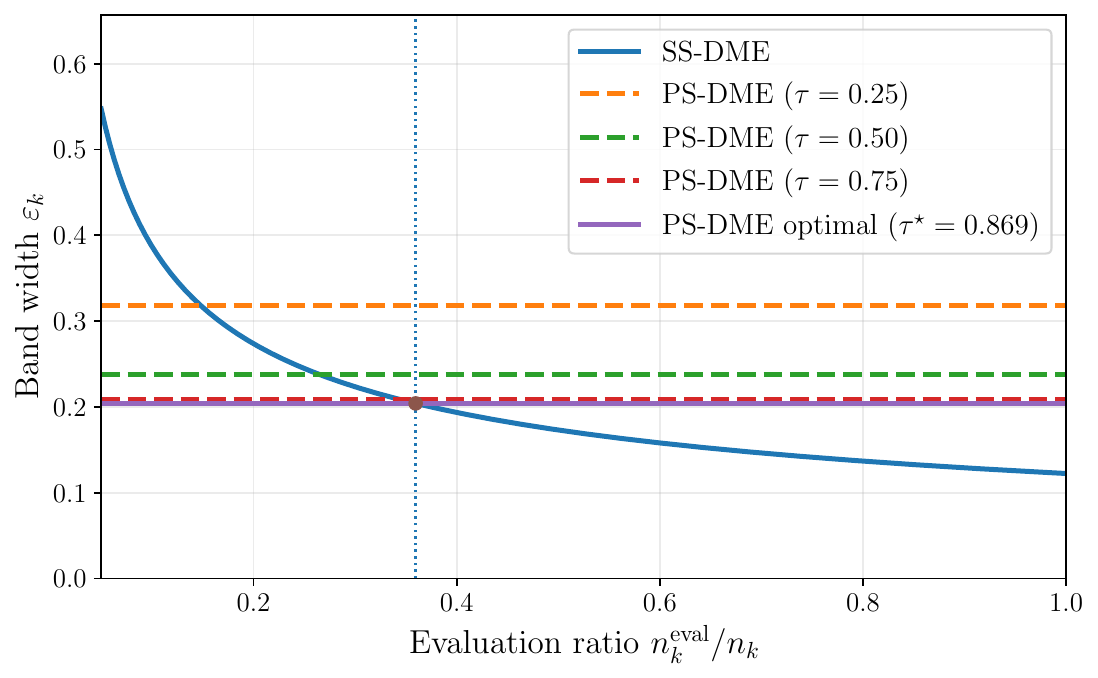}
    \caption{Comparison of the band widths produced by SS-DME and by PS-DME with the e-calibrator family (\ref{eq:power_e-calibrator}). The parameter $\delta |\mathcal{K}|/K$ is fixed to 0.01.}
    \label{fig:width_analysis}
\end{figure}

\section{Additional Related Work}
\label{app:related_work}

\paragraph{Reliable hyperparameter selection:}
Recent work has shown that hyperparameter selection can be cast as an MHT problem, enabling finite-sample guarantees for the selected model \citep{angelopoulos2025learn}. This framework has been extended to a variety of performance criteria and engineering applications \citep{stephen2021gentle, bates2023testing, angelopoulos2022conformal}; see \citep{farzaneh2025ensuring} for a recent review. A key limitation of this literature is that it focuses on scalar performance metrics, e.g., average risk, relative to a pre-specified target level. In contrast, this work aims at \emph{distributional} evaluation, enabling comparison of models across all test-time reliability levels without committing to a single KPI threshold.

\paragraph{Confidence bands for distribution functions:}
Classical nonparametric methods provide confidence bands for a single distribution function, notably via the DKW inequality \citep{dvoretzky1956asymptotic, massart1990tight}. Likelihood-ratio and goodness-of-fit approaches, such as Berk--Jones and empirical likelihood methods (see Appendix \ref{app:generalization}), can yield tighter or more adaptive bands \citep{berk1979goodness, owen1995nonparametric, dumbgen2023new}. However, these methods are designed for marginal inference and do not account for model selection. In contrast, our setting requires validity after arbitrary data-dependent screening.

\section{Proof of Lemma \ref{theorem:validity_split}}
\label{app:proof_lemma_1}

For each fixed $k$,
the DKW inequality (\ref{eq:dkw_split}) guarantees the inequality
\begin{equation}
\mathbb P\!\left(
F_k \notin \mathcal C_k
\right)
\le
\delta.
\end{equation}
Therefore, we have
\begin{subequations}\label{eq:fcr}
\begin{align}
\mathrm{FCR}
&=
\E_\mathcal{K}\left[\E\!\left[
\frac{1}{\max\{|\mathcal K|,1\}}
\sum_{k \in \mathcal K}
\mathbf{1}\{F_k \notin \mathcal C_k\}
\middle | \mathcal{K}\right]\right] \label{eq:fcr:a} \\
& = \E_\mathcal{K}\left[\frac{1}{\max\{|\mathcal K|,1\}}\E\!\left[
\sum_{k \in \mathcal K}
\mathbf{1}\{F_k \notin \mathcal C_k\}
\middle | \mathcal{K}\right]\right] \label{eq:fcr:b}\\
&\le
\E_\mathcal{K}\!\left[
\frac{1}{\max\{|\mathcal K|,1\}}
\sum_{k \in \mathcal K}
\delta
\right] \label{eq:fcr:c} \\
&\leq
\delta \label{eq:fcr:d}
\end{align}
\end{subequations}
where the outer expectation $\E _\mathcal{K}[\cdot]$ is with respect to the pre-selected subset $\mathcal{K}$,
which establishes the guarantee.

\section{Proof of Theorem \ref{theorem:validity_full}}
\label{app:theorem_proof}

For any candidate CDF $F$,
consider the null hypothesis
\begin{equation}
\label{eq:insample_null}
\mathcal H_{0,k}(F):
\quad
F_k = F
\end{equation}
that the candidate distribution $F$
coincides with the true distribution $F_k$. Define the Kolmogorov distance \citep{wasserman2006all}
\begin{equation}
\label{eq:T_insample}
T_k(F)
=
\sup_{x \in \mathbb R}
\left|
\widehat F_k(x) - F(x)
\right|.
\end{equation}

By the DKW inequality \citep{dvoretzky1956asymptotic},
under the null hypothesis \eqref{eq:insample_null}, the inequality
\begin{equation}
\mathbb P\!\left(
T_k(F)
>
\varepsilon \mid \mathcal{H}_{0,k}(F)
\right)
\le
2 \exp\!\left(-2 n_k \varepsilon^2 \right)
\end{equation}
holds. As shown next, this implies that the following statistic
\begin{equation}
\label{eq:pval_insample}
p_k(F)
=
\min\Bigl\{
1,\,
2 \exp\!\bigl(-2 n_k (T_k(F))^2\bigr)
\Bigr\}
\end{equation}
is a valid p-value for the null hypothesis (\ref{eq:insample_null}).

\begin{lemma}
\label{lem:pval_valid_insample}
For the null hypothesis $\mathcal H_{0,k}(F)$ in \eqref{eq:insample_null},
the statistic $p_k(F)$ in \eqref{eq:pval_insample}
is a valid p-value, i.e.,
for all $\alpha \in (0,1)$,
\begin{equation}
\mathbb P\!\left(
p_k(F) \le \alpha
\,\middle|\,
\mathcal H_{0,k}(F)
\right)
\le
\alpha.
\end{equation}
\end{lemma}

\begin{proof}
Fix $\alpha \in (0,1)$.
The event $\{p_k(F) \le \alpha\}$
is equivalent to
the inequality
\begin{equation}
T_k(F)
\ge
\sqrt{
\frac{1}{2 n_k}
\log\!\left(\frac{2}{\alpha}\right)
}.
\end{equation}
Under hypothesis $\mathcal H_{0,k}(F)$,
the DKW inequality \citep{dvoretzky1956asymptotic} implies
\begin{equation}
\mathbb P\!\left(
T_k(F)
\ge
\sqrt{
\frac{1}{2 n_k}
\log\!\left(\frac{2}{\alpha}\right)
}
\middle| \mathcal{H}_{0,k}(F)\right)
\le
2 \exp\!\left(
-2 n_k \cdot
\frac{1}{2 n_k}
\log\!\frac{2}{\alpha}
\right)
=
\alpha,
\end{equation}
which establishes super-uniformity and concludes the proof.
\end{proof}

Given the p-value $p_k(F)$ in \eqref{eq:pval_insample},
we construct the statistic
\begin{equation}
E_k(F)
=
f\!\left(p_k(F)\right),
\label{eq:e_val_general_new}
\end{equation}
which is a valid e-value, i.e.,
\begin{equation}
\E\!\left[E_k(F)\mid \mathcal{H}_{0,k}(F)\right] \le 1.
\end{equation}

The confidence bound $C_k = \{F:L_k(x)\le F(x)\le U_k(x)\;\;\text{for all}\;\;x\in \R\}$ with (\ref{eq:e_lower}) and (\ref{eq:e-upper}) produced by PS-DME is obtained by inverting the binary test of the null hypothesis $\mathcal{H}_{0,k}(F)$ using the e-value $E_k(F)$ as
\begin{equation}
\mathcal C_k
=
\left\{
F:\ 
E_k(F)
<
\frac{K}{\delta |\mathcal{K}|}
\right\} = \left\{ F:  p_k(F)
>
f^{-1}\!\left(\frac{K}{\delta |\mathcal{K}|}\right) \right\}.
\label{eq:eby_threshold_general_new}
\end{equation}

Let $\mathcal O_k$ denote the miscoverage event (\ref{eq:not_covered}).
By construction of set (\ref{eq:eby_threshold_general_new}), we have
\begin{equation}
\ind\{\mathcal O_k\}
=
\ind\!\left\{
E_k(F_k)
\ge
\frac{K}{\delta |\mathcal{K}|}
\right\}
\le
\frac{\delta |\mathcal{K}|}{K} E_k(F_k),
\end{equation}
where $F_k$ represents the true CDF, and the inequality follows from the condition $\ind \{x\ge a \}\le x/a$, which holds for all $x\ge 0$ and $a>0$.
Therefore, taking the expectation with respect to the true joint distribution of the KPIs, the FCR in (\ref{eq:def:fcr}) is evaluated as 
\begin{subequations}\label{eq:fcr2}
\begin{align}
\mathrm{FCR}
&=
\E\!\left[
\frac{1}{\max\{|\mathcal{K}|,1\}}
\sum_{k\in \mathcal{K}}
\ind\{\mathcal O_k\}
\right] \label{eq:fcr2:a} \\
&\le
\E\!\left[
\frac{1}{\max\{|\mathcal{K}|,1\}}
\sum_{k\in \mathcal{K}}
\frac{\delta |\mathcal{K}|}{K}E_k(F_k)
\right] \label{eq:fcr2:b} \\
&\le
\frac{\delta}{K}
\sum_{k=1}^K
\E[E_k(F_k)] \label{eq:fcr2:c}\\
& = \frac{\delta}{K}\sum_{k = 1}^K\mathbb{E}\left[E_k(F_k)\middle| \mathcal{H}_{0,k}(F_k)\right] \le \delta \label{eq:fcr2:d}
\end{align}
\end{subequations}

\section{Proof of Corollary \ref{cor:vs_condition}}
\label{app:cor_proof}

The Vovk--Sellke bound \citep{vovk2021values} states that for the power
family of calibrators (\ref{eq:power_e-calibrator}) we have
\begin{equation}
\label{eq:fvs}
f_{\mathrm{VS}}(p)
=
\max_{\tau\in(0,1)} f_\tau(p)
=
\begin{cases}
\dfrac{-\exp(-1)}{p \log p}, & p \le e^{-1}, \\[6pt]
1, & p > e^{-1},
\end{cases}
\end{equation}
where
\begin{equation}
\label{eq:tau_argmax}
    \arg\max_{\tau\in(0,1)} f_\tau(p) = -\frac{1}{\log p}
\end{equation}
Hence, for all $p\in(0,1)$ and $\tau\in(0,1)$, the inequality
\begin{equation}
\label{eq:f_tau}
f_\tau(p) \le f_{\mathrm{VS}}(p)
\end{equation}
holds. Note, however, that the function $f_{\mathrm{VS}}(p)$ is not a valid e-calibrator.

Since function $f_\tau(\cdot)$ is decreasing, the corresponding inverse satisfies the inequality
$f_\tau^{-1}(y)
\le
f_{\mathrm{VS}}^{-1}(y)$
for all $y \ge 1$.
Furthermore, the inverse of the Vovk-Sellke bound is given by
\begin{equation}
\label{eq:vs_proof}
f_{\mathrm{VS}}^{-1}(y)
=
\exp\!\left(
W\!\left(-\frac{1}{ey}\right)
\right)
\end{equation}
for $p\leq e^{-1}$.
Substituting $y=K/(\delta|\mathcal K|)$ in \eqref{eq:vs_proof} and plugging 
the result into the denominator of \eqref{eq:compare_width_general_vs_split_rho}, 
with $\log(2/\delta)$ in the numerator, yields the condition \eqref{eq:vs_condition}. Given (\ref{eq:f_tau}), the inequality (\ref{eq:compare_width_general_vs_split_rho}) for any $\tau$ implies (\ref{eq:vs_condition}), proving the necessity of condition (\ref{eq:vs_condition}). Furthermore, using (\ref{eq:tau_argmax}), we obtain that the maximizing choice of the power-calibrator parameter at level $y=K/(\delta|\mathcal K|)$ is \eqref{eq:optimal_calibrator}. Therefore, when $|\mathcal K|$ is fixed, condition \eqref{eq:vs_condition} ensures that this choice of $\tau$ makes \eqref{eq:compare_width_general_vs_split_rho} hold, completing the proof.

\section{Alternative Implementation of PS-DME via Berk--Jones CDF Bands}
\label{app:generalization}

The proof of Theorem~\ref{theorem:validity_full} in Appendix~\ref{app:theorem_proof}
shows that PS-DME applies to any valid test of the null hypothesis
$\mathcal H_{0,k}(F)$ in \eqref{eq:insample_null}. While Sec.~\ref{sec:method}
uses the Kolmogorov statistic $T_k(F)$ in \eqref{eq:T_insample} together with
the DKW inequality \eqref{eq:dkw_split}, one may instead use likelihood-ratio-type goodness-of-fit
statistics such as the Berk--Jones (BJ) statistic \citep{berk1979goodness}.

Define the BJ statistic for a candidate CDF $F$ as
\begin{equation}
\label{eq:bj_stat_short}
S_{n_k}(F)
=
\sup_{x \in \R}
K\!\bigl(\widehat F_k(x),\,F(x)\bigr),
\end{equation}
where $K(\cdot,\cdot)$ is the Bernoulli KL divergence defined as
\begin{equation}
\label{eq:kl_bernoulli_app}
K(a,b)
=
a\log\!\left(\frac{a}{b}\right)
+
(1-a)\log\!\left(\frac{1-a}{1-b}\right),
\qquad a,b\in(0,1).
\end{equation}

Under the null hypothesis $\mathcal H_{0,k}(F)$ and assuming that $F$ is continuous,
the probability integral transform implies that the transformed variables
$U_{k,i} = F(X_{k,i})$ are i.i.d.\ $\mathrm{Unif}[0,1]$
\citep{shorack2009empirical}. Let $\widehat G$ denote the empirical CDF of
$\{U_{k,i}\}_{i=1}^{n_k}$. Then, under the null, the statistic $S_{n_k}(F)$ in
\eqref{eq:bj_stat_short} has the same distribution as
\begin{equation}
\label{eq:bj_null_stat}
S_{n_k}^{(0)}
=
\sup_{u \in [0,1]}
K\!\bigl(\widehat G(u),\,u\bigr),
\end{equation}
i.e., the BJ statistic computed from an i.i.d.\ sample drawn from
$\mathrm{Unif}[0,1]$.

Let $q_{n_k}(1-\alpha)$ denote the $(1-\alpha)$ quantile of the null distribution of random variable $S_{n_k}^{(0)}$ in \eqref{eq:bj_null_stat}. Then, by test
inversion \citep{owen1995nonparametric,jager2007goodness}, the acceptance region
\begin{equation}
S_{n_k}(F) \le q_{n_k}(1-\alpha)
\end{equation}
defines a valid $(1-\alpha)$ confidence set for the true CDF $F_k$. This condition is
equivalent to requiring that, for all $x \in \R$, the value
$F(x)$ lies in the set of all $u \in [0,1]$ satisfying
\begin{equation}
K\!\bigl(\widehat F_k(x),\,u\bigr)
\le
q_{n_k}(1-\alpha).
\end{equation}
Accordingly, the lower and upper bounds at each $x$ are obtained by solving
this inequality for $u$, yielding confidence bands
\begin{equation}
L_k(x) \le F_k(x) \le U_k(x), \qquad \forall x \in \R,
\end{equation}
which depend on the local value $\widehat F_k(x)$ and are generally asymmetric. To guarantee post-selection FCR control, the nominal level is set to $\alpha = f^{-1}\!\left(K/(\delta|\mathcal{K}|)\right)$, consistent with 
the e-value inversion in \eqref{eq:eby_threshold_general_new}.

Compared to the DKW-based bands in (\ref{eq:e_lower})--(\ref{eq:e-upper}),
whose width is the input-independent quantity $\varepsilon_k$ in
\eqref{eq:width_general_new}, the BJ bands are input-dependent. In particular,
they are typically tighter in the tails and wider around the center of the CDF,
reflecting the geometry of the KL divergence: deviations are penalized more
strongly when $\widehat F_k(x)$ is close to $0$ or $1$, and more weakly near
$\widehat F_k(x)\approx 1/2$ \citep{jager2007goodness}. This improved adaptivity comes at the cost of
computing the critical value $q_{n_k}(1-\alpha)$, e.g., via No\'e's recursion
\citep{noe1972calculation}, and solving a one-dimensional root-finding problem
for each value of $\widehat F_k(x)$.

We finally compare the default DKW-based implementation of PS-DME in
(\ref{eq:e_lower})--(\ref{eq:e-upper}) with the BJ-based implementation
described above. Fig.~\ref{fig:bj_vs_dkw_synth}(a) shows the confidence bands
obtained by the two methods for a representative hyperparameter. The BJ-based
band is tighter in the tails of the CDF and wider around the center, in
agreement with the discussion above.

The impact of this difference on model evaluation is illustrated in
Fig.~\ref{fig:bj_vs_dkw_synth}(b), which reports the post-selection best
guaranteed KPI as a function of the
target outage probability $\gamma$. The BJ-based implementation provides a
stronger guarantee for most outage levels, particularly in the tail regimes.
In contrast, for intermediate outage probabilities, the DKW-based
implementation can yield slightly tighter guarantees, reflecting its uniform
band width in \eqref{eq:width_general_new}.

\begin{figure}[t]
    \centering
    \begin{subfigure}[t]{0.48\linewidth}
        \centering
        \includegraphics[width=\linewidth]{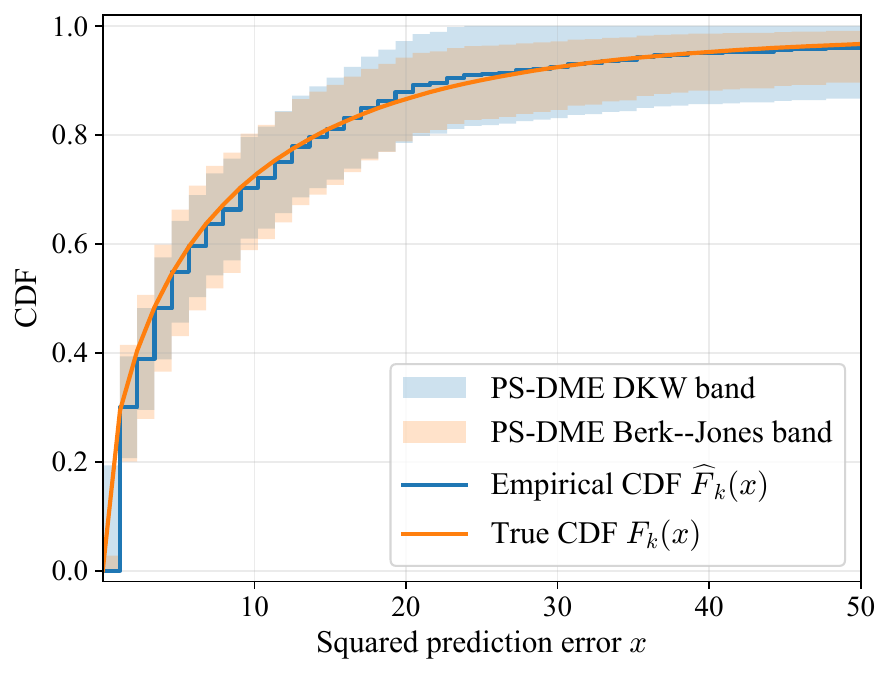}
        \caption{}
    \end{subfigure}
    \hfill
    \begin{subfigure}[t]{0.48\linewidth}
        \centering
        \includegraphics[width=\linewidth]{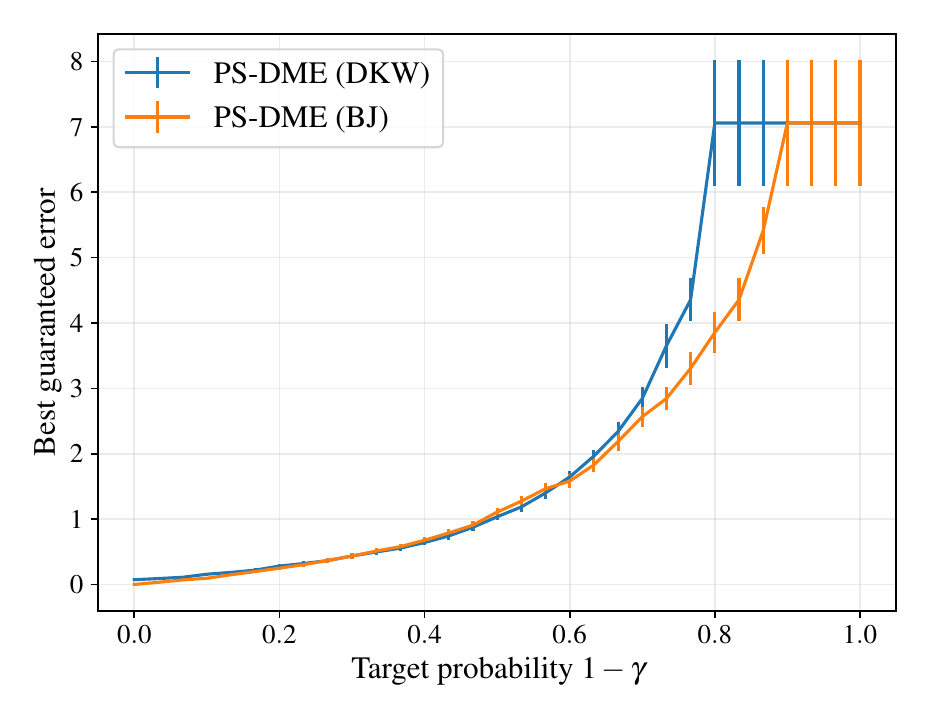}
        \caption{}
    \end{subfigure}
    \caption{
    Comparison between the DKW-based and Berk--Jones-based implementations of PS-DME.
    (a) CDF confidence bands for a representative hyperparameter;
    (b) Best guaranteed KPI as a function of the outage probability $\gamma$.
    }
    \label{fig:bj_vs_dkw_synth}
\end{figure}

\section{Additional Experiments}

\subsection{Empirical FCR for the Synthetic Data Experiment}
\label{app:fcr_table}

Table~\ref{tab:fcr} reports the empirical FCR for SS-DME and PS-DME 
on the synthetic experiment of Sec.~\ref{sec:synthetic}, 
averaged over 100 random replications at target level $\delta = 0.1$.
SS-DME uses a split ratio $n^\mathrm{sel}/n = 0.5$.
Both methods satisfy the FCR requirement (\ref{eq:fcr_guarantee}).

\begin{table}[h!]
\centering
\caption{Empirical FCR for the synthetic experiment at $\delta = 0.1$, 
averaged over 100 replications.}
\label{tab:fcr}
\begin{tabular}{llc}
\toprule
$n$ & Method & Empirical FCR \\
\midrule
\multirow{2}{*}{20}  & SS-DME  & 0.0401 \\
                     & PS-DME  & 0.0116 \\
\midrule
\multirow{2}{*}{100} & SS-DME  & 0.0923 \\
                     & PS-DME  & 0 \\
\bottomrule
\end{tabular}
\end{table}

\subsection{Evaluating and Selecting Telecom Network Configurations}
\label{sec:exp_oran}

We evaluate the proposed PS-DME on a wireless networking dataset derived from an
Open Radio Access Network (O-RAN) deployment \citep{bonati2021orancommag}.

\subsubsection{Dataset and system setup}

We use the dataset introduced in
\citep{bonati2021orancommag}, which was collected on the Colosseum wireless
network emulator at Northeastern University \citep{polese2023understanding}.
The dataset contains measurements from an O-RAN deployment consisting of
four base stations operating with a channel bandwidth of $3$ MHz, consisting of
$15$ physical resource blocks (PRBs), and serving $40$ user equipments (UEs).
Traffic is organized into three slices representing distinct service
classes, namely enhanced mobile broadband (eMBB), massive machine-type communications (MTC), and ultra-reliable low-latency communications (URLLC) \citep{popovski20185g}. The dataset reports periodic slice-level measurements, including the
downlink buffer occupancy and downlink transmission rate for each slice, across different network configurations, which specify (\textit{i}) the scheduling policy used for each slice; and (\textit{ii}) the initial allocation of radio resources, i.e., PRBs, across slices. Specifically, the dataset specifies $18$ predefined configurations, each corresponding to one particular combination of one of three schedulers and of one of six initial allocations for the three slices. We adopt the downlink
buffer occupancy for the URLLC slice as the (negatively
oriented) KPI $X_{k,i}$. The ground truth CDF of this KPI is estimated for each configuration by using
$90\%$ of the available data.

Model pre-selection is performed by selecting
the six configurations with the smallest empirical mean buffer occupancy for the URLLC slice.

\subsubsection{Results}

Fig.~\ref{fig:oran_cdf_band} illustrates representative post-selection
CDF bands for a selected O-RAN configuration for SS-DME and PS-DME, with SS-DME using half of the calibration data for pre-selection.
The horizontal axis corresponds to the buffer occupancy KPI, while the vertical axis shows the corresponding
CDF.
As in the previous experiments, SS-DME yields
wider bands compared to PS-DME. Additionally, because SS-DME uses fewer calibration samples to estimate the empirical CDF, its empirical CDF exhibits larger deviations from the proxy true CDF than that of PS-DME. Note also that, since SS-DME uses a subset of the calibration dataset for evaluating the empirical CDF, the empirical CDFs for SS-DME and PS-DME are different.

\begin{figure}[t]
\centering
\includegraphics[width=0.75\linewidth]{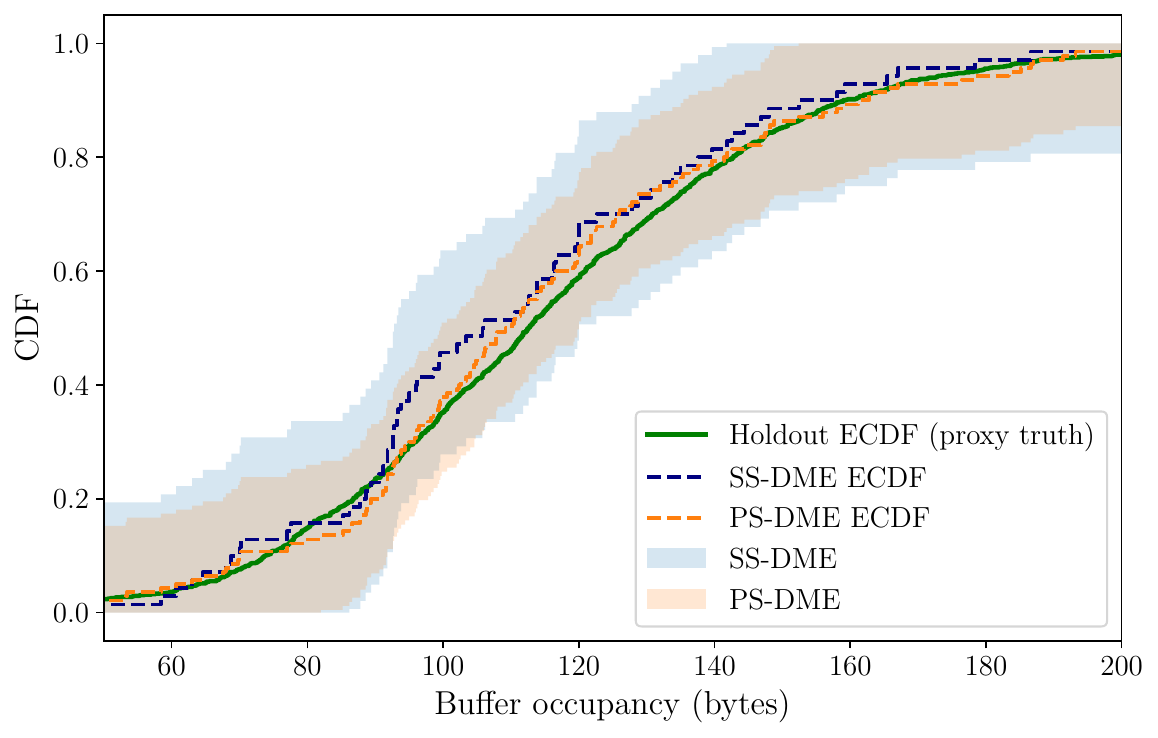}
\caption{
Representative post-selection CDF bands for the O-RAN buffer occupancy
experiment for a selected configuration for PS-DME, and SS-DME with dataset split ratio $|\mathcal{D}^\text{sel}|/|\mathcal{D}| = 0.5$.
}
\label{fig:oran_cdf_band}
\end{figure}

We then analyze the post-selection guaranteed KPI.
Fig.~\ref{fig:oran_best_quantile} reports the best guaranteed buffer occupancy
as a function of the target probability $1-\gamma$ for PS-DME, as well as SS-DME with dataset split ratios $n^\text{sel}/n$ from the set $\{0.2, 0.3, 0.4\}$.
The proposed PS-DME consistently yields stronger
guarantees than the SS-DME variants.

\begin{figure}[t]
\centering
\includegraphics[width=0.6\linewidth]{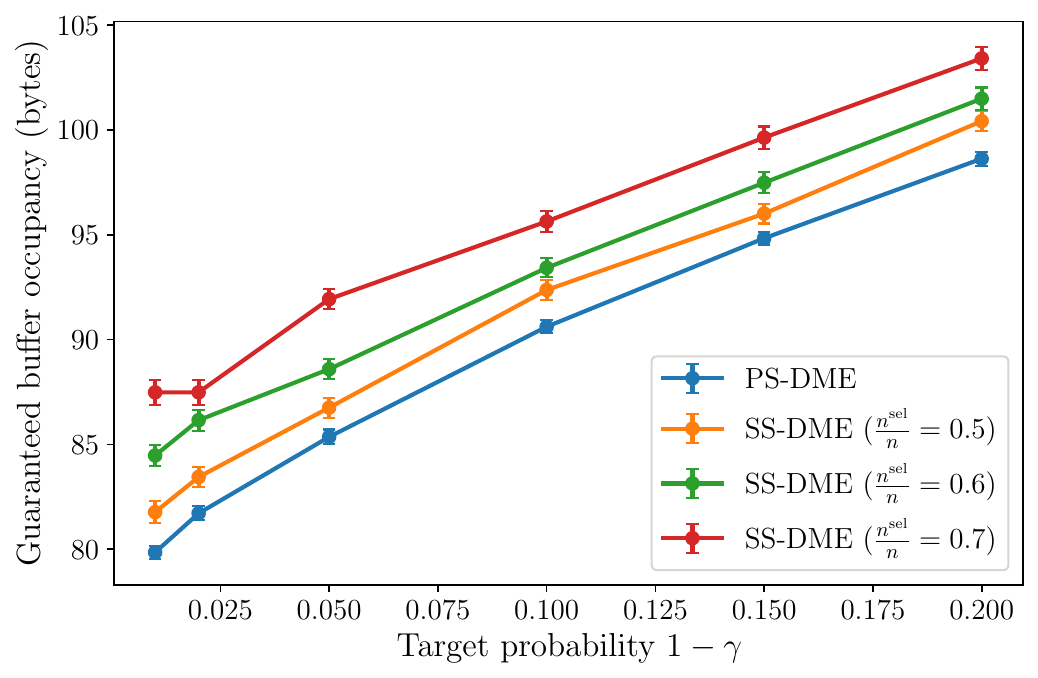}
\caption{
Best guaranteed buffer occupancy for the O-RAN experiment as a function of
the target probability $1-\gamma$ for PS-DME and SS-DME with dataset split ratio $n^\text{sel}/n\in\{0.2, 0.3, 0.4\}$. The vertical error bars show the standard error of the mean across 100 random splits of the dataset.
}
\label{fig:oran_best_quantile}
\end{figure}

To further investigate the trade-off induced by sample splitting, we fix two target
probabilities, $1-\gamma \in \{0.7, 0.8\}$, and vary the split ratio
$n^{\mathrm{sel}}/n$ used by SS-DME. Fig.~\ref{fig:oran_split_tradeoff} reports the resulting best guaranteed buffer
occupancy as a function of the split ratio. The two panels correspond to different
calibration regimes: Fig.~\ref{fig:oran_split_tradeoff}(a) uses $30\%$ of
the available data for calibration, while Fig.~\ref{fig:oran_split_tradeoff}(b)
uses the entire available dataset. In each panel, we report results for both
target probabilities $1-\gamma = 0.7$ and $1-\gamma = 0.8$.

In the low-calibration regime shown in Fig.~\ref{fig:oran_split_tradeoff}(a),
SS-DME consistently yields a larger best guaranteed buffer occupancy than
PS-DME across all split ratios and both target probabilities. This reflects the
fact that splitting an already limited calibration dataset reduces the effective
sample size available for both model selection and inference, leading to more
conservative guarantees and hence worse (larger) guaranteed KPI values. In contrast, in the high-calibration regime shown in
Fig.~\ref{fig:oran_split_tradeoff}(b), the trade-off becomes more nuanced.
For intermediate values of the split ratio, SS-DME can outperform PS-DME for the
less stringent target $1-\gamma = 0.7$, achieving a smaller guaranteed buffer
occupancy. This improvement arises because allocating a moderate portion of the
data to selection can lead to better configuration choices, while the remaining
data are still sufficient to yield tight confidence bands. However, this advantage
does not persist uniformly across all split ratios or for the more stringent target
$1-\gamma = 0.8$, where the need for precise distributional estimation favors
PS-DME, which leverages the full dataset.

\begin{figure}[t]
\centering
\begin{subfigure}[t]{0.48\linewidth}
    \centering
    \includegraphics[width=\linewidth]{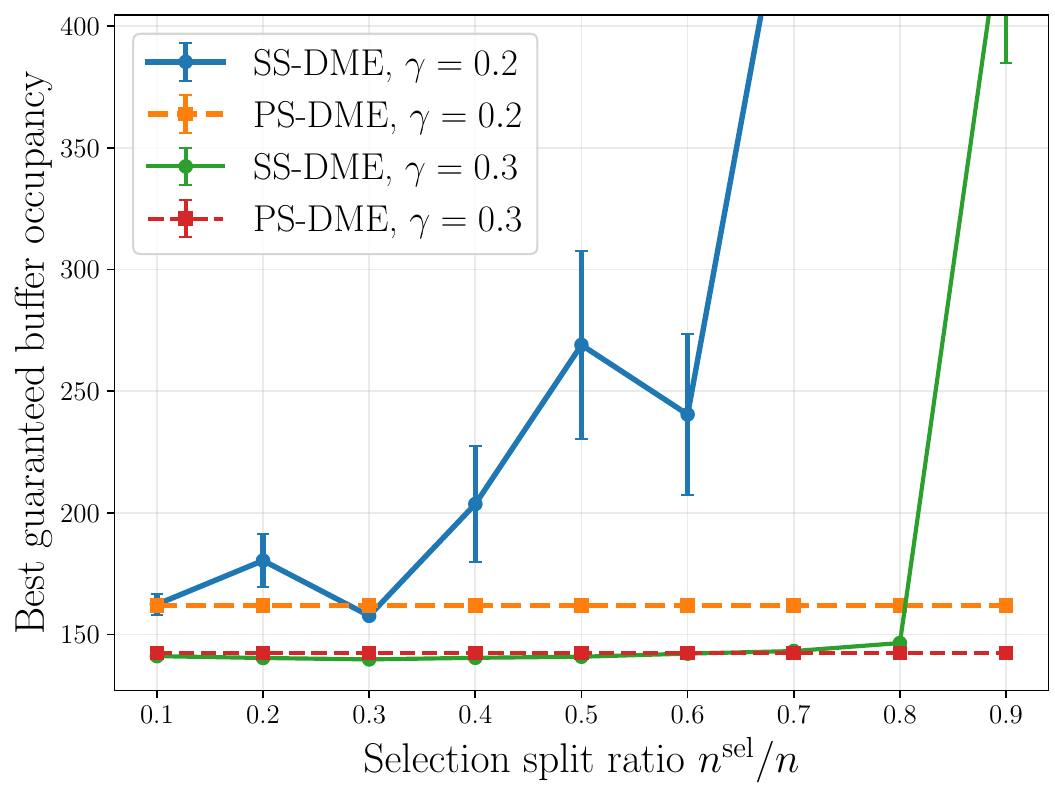}
    \caption{}
    \label{fig:oran_split_tradeoff_kpi}
\end{subfigure}
\hfill
\begin{subfigure}[t]{0.48\linewidth}
    \centering
    \includegraphics[width=\linewidth]{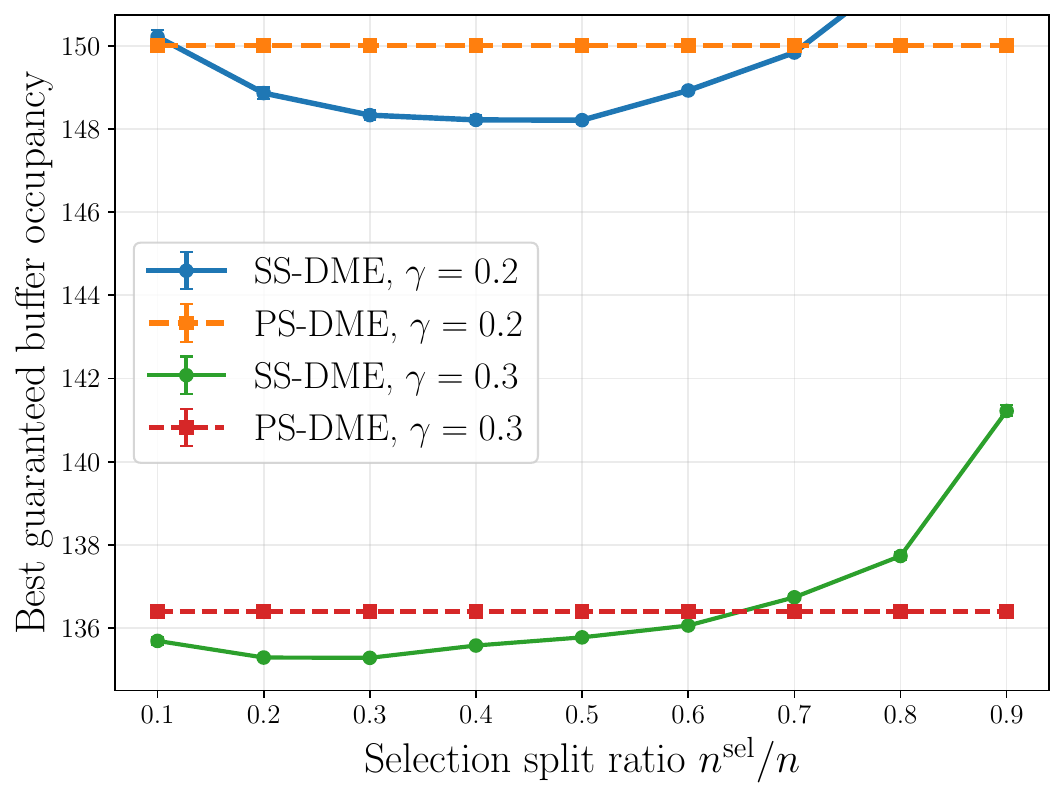}
    \caption{}
    \label{fig:oran_split_tradeoff_width}
\end{subfigure}
\caption{
Effect of the SS-DME split ratio $n^\text{sel}/n$
on the best guaranteed buffer occupancy in the O-RAN experiment, for two target
probabilities $1-\gamma \in \{0.7, 0.8\}$. Each panel corresponds to a different calibration regime, with (a) $30\%$ and (b) $100\%$ of the available data used for calibration. The vertical error bars show the standard error of the mean across 1,000 random splits of the calibration dataset into pre-selection and evaluation datasets.
}
\label{fig:oran_split_tradeoff}
\end{figure}

\subsection{Post-Selection CDF Inference for Server Latency}
\label{sec:real_logs}

We next evaluate the proposed PS-DME confidence bands on response-time distributions extracted from HTTP
server access logs.
Latency is a canonical KPI in web services and
distributed systems, where configurations are routinely compared and selected
based on empirical performance measured on the same data
\citep{dean2013tail,barroso2019datacenter}.
This setting provides a natural testbed for post-selection CDF inference, as
distributional information, rather than a single scalar metric, is often of
primary interest.

\subsubsection{Setup}
We consider a collection of HTTP server access logs obtained from a publicly available server log
dataset released on Kaggle and commonly used for benchmarking log analysis
and anomaly detection methods \citep{kaggle_server_logs}.
Each log entry records, among other fields, the HTTP request method, the
request path, and the response time measured at the server.
In this experiment, the hyperparameters correspond to discrete
request-handling configurations defined by ordered pairs
$(\text{HTTP method}, \text{request path})$.
The HTTP method takes values in
$\{\texttt{GET}, \texttt{POST}, \texttt{PUT}, \texttt{DELETE}\}$, which comprise
all methods observed in the dataset.
The request path is defined as the URL path component of the request line, with
query parameters removed, and takes values in the finite set of distinct paths
observed across the log file.

To construct a finite hyperparameter grid, we enumerate all method--path pairs
appearing in the dataset and compute their total frequency.
We retain only those pairs that appear at least $N_{\min} = 200$ times, ensuring that
each candidate configuration admits sufficiently many samples for stable
empirical CDF estimation.
Among the remaining pairs, we select the $K = 30$ most frequent configurations to
form a fixed hyperparameter grid, which is determined once and held constant
across all experimental repetitions.
Each configuration $k$ in this grid induces an unknown latency distribution
$P_k$ corresponding to the response times of requests matching the associated
method--path pair.

Within each experimental repetition, the data corresponding to the fixed
hyperparameter grid are randomly split into a calibration set and a holdout set.
Specifically, $30\%$ of the data are allocated to calibration and the
remaining $70\%$ to holdout.
Using the calibration data, we compute the empirical mean latency for each
configuration and select a subset
$\mathcal{K} \subseteq \{1,\dots,K\}$
consisting of the $M = 12$ configurations with the smallest empirical mean latency.
This selection rule is fully data-dependent and reuses the same samples later
employed to construct confidence bands.

For each selected configuration $k \in \mathcal{K}$, we construct two-sided uniform CDF
confidence bands using the methodology described in
Section~\ref{sec:method}. To study the effect of the e-calibrator on both inferential tightness and
post-selection reliability, we compare three members of the power-family
of e-calibrators (\ref{eq:power_e-calibrator}),
and consider three parameters
$\tau \in \{1/3,\; 1/2,\; 2/3\}$.
For each calibrator, the corresponding band width is computed according to
\eqref{eq:width_general_new}, and the same selected set $\mathcal{K}$ is used across all
calibrators within each repetition to enable a fair comparison.

\subsubsection{Results}
Because the true latency distributions are unknown, we assess post-selection
validity using the empirical CDF computed on the holdout set as a proxy for the
true CDF.
A band is declared to fail if the holdout empirical CDF exits the proposed
interval at any point.
We repeat the entire procedure over 300 independent random splits and compute
the FCP for each repetition.

\begin{figure}[h!]
    \centering
    \includegraphics[width=0.75\linewidth]{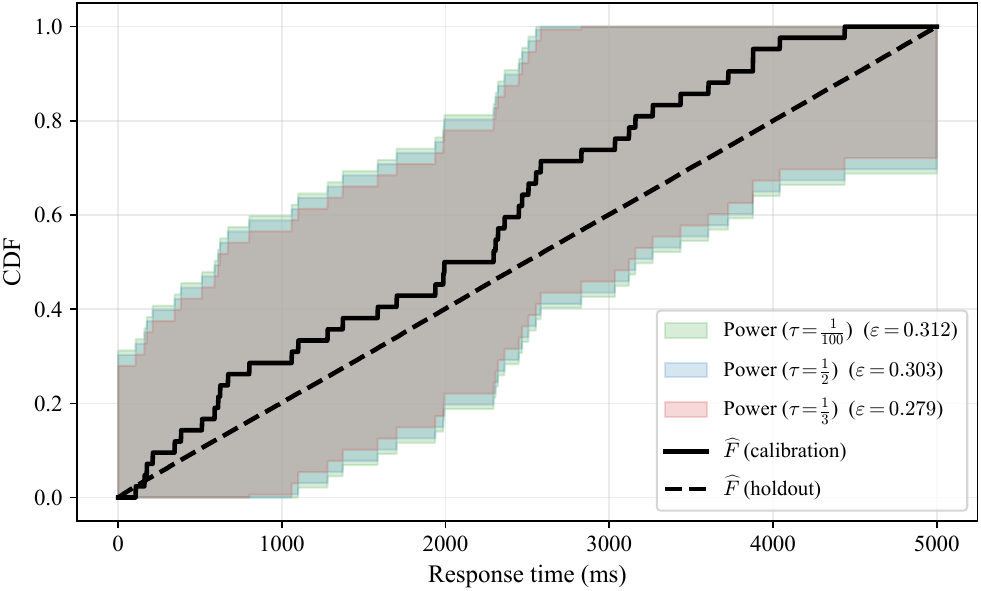}
    \caption{Representative post-selection CDF bands for a selected server
    configuration using three different e-calibrators.
    The empirical CDF is common across all curves, while the shaded regions
    correspond to different values of parameter $\tau$ in (\ref{eq:power_e-calibrator}).}
    \label{fig:cdf_band_real}
\end{figure}

Fig.~\ref{fig:cdf_band_real} compares post-selection CDF bands constructed using
three different e-calibrators for a representative selected configuration.
All curves share the same empirical CDF, while the shaded regions correspond to
the confidence bands induced by the different values of $\tau$. The resulting band widths vary with the choice of calibrator.
In our experiment, the tightest band among the tested calibrators is obtained
for $\tau = 1/3$, while both larger and smaller values of $\tau$ produce
slightly wider bands. This behavior reflects the trade-off introduced by the
e-calibrator in the inversion of the e-value tests: different calibrators
allocate evidence differently across the range of $p$-values, which in turn
affects the resulting confidence radius in \eqref{eq:width_general_new}.
These results illustrate that the choice of calibrator can have a noticeable
effect on inferential tightness, and that intermediate values of $\tau$
may provide the most informative bands in practice.


{
\subsection{Empirical FCR on Real-World Benchmarks}
\label{app:fcr_realworld}

To complement the empirical FCR table of Appendix~\ref{app:fcr_table}, we report the empirical FCR of PS-DME on the two real-world benchmarks of Sec.~\ref{sec:exp_spider_llm} and Appendix~\ref{sec:exp_oran}. For the Spider text-to-SQL experiment, we use $n_{\mathrm{cal}}=100$ calibration samples and a held-out set of $100$ samples as a proxy for the true distribution; results are averaged over 500 independent random splits. For the O-RAN experiment, we use the same split protocol as in Appendix~\ref{sec:exp_oran} and average over 300 random splits. In both cases, the target level is $\delta = 0.1$.

\begin{table}[h!]
\centering
\caption{{Empirical FCR of PS-DME on the Spider text-to-SQL and O-RAN telecom benchmarks at target $\delta = 0.1$. Mean $\pm$ one standard error of the mean across random splits.}}
\label{tab:fcr_realworld}
\begin{tabular}{llc}
\toprule
Benchmark & Setting & Empirical FCR \\
\midrule
\multirow{3}{*}{Spider (text-to-SQL)} & $M = 5$ selected  & $0.031 \pm 0.007$ \\
                                       & $M = 10$ selected & $0.046 \pm 0.007$ \\
                                       & $M = 20$ selected & $0.070 \pm 0.008$ \\
\midrule
O-RAN (telecom)                       & $M = 6$ selected  & $0.007 \pm 0.001$ \\
\bottomrule
\end{tabular}
\end{table}

As shown in Table~\ref{tab:fcr_realworld}, PS-DME satisfies the FCR requirement (\ref{eq:fcr_guarantee}) in all settings, even on the O-RAN benchmark where the i.i.d.\ assumption is only approximate (see Sec.~\ref{sec:limitations}). On the O-RAN dataset, $92\%$ of random repetitions exhibit zero miscoverage, indicating substantial slack with respect to the nominal target.

\subsection{Selection-Ratio Sweep}
\label{app:selection_sweep}

The experiments in Sec.~\ref{sec:experiments} use a fixed selected-set size $M$. In this appendix, we vary $M$ across its full operating range to characterize the dependence of the band width and the guaranteed KPI on the selection ratio $M/K$. We use the synthetic data generating process of Sec.~\ref{sec:synthetic} with $K = 31$ configurations, $n = 200$, $\delta = 0.1$, and $300$ random repetitions. For each value of $M$, SS-DME is given the number of evaluation samples $n^{\mathrm{eval}}$ equal to $70\%$ of the Corollary~\ref{cor:vs_condition} crossover threshold, i.e., a regime in which Corollary~\ref{cor:vs_condition} guarantees that PS-DME is narrower by design. We report the band width $\varepsilon$ and the best guaranteed KPI $X^\star$ at target reliability $1-\gamma = 0.75$ computed via the lower band and \eqref{eq:best_guaranteed_k}.

\begin{table}[h!]
\centering
\caption{{Band width $\varepsilon$ and best guaranteed KPI $X^\star$ at $1-\gamma = 0.75$ for PS-DME and SS-DME on synthetic data, as a function of the selected-set size $M$ (with $K = 31$). Lower is better for both $\varepsilon$ and $X^\star$.}}
\label{tab:selection_sweep}
\begin{tabular}{cccccc}
\toprule
$M$ & $M/K$ & $\varepsilon_{\mathrm{PS}}$ & $\varepsilon_{\mathrm{SS}}$ & $X^\star_{\mathrm{PS}}$ & $X^\star_{\mathrm{SS}}$ \\
\midrule
2  & 0.06 & 0.149 & 0.179 & 0.793 & 0.814 \\
5  & 0.16 & 0.139 & 0.168 & 0.791 & 0.807 \\
10 & 0.32 & 0.132 & 0.158 & 0.790 & 0.804 \\
15 & 0.48 & 0.127 & 0.153 & 0.789 & 0.800 \\
20 & 0.65 & 0.124 & 0.148 & 0.789 & 0.799 \\
25 & 0.81 & 0.121 & 0.145 & 0.789 & 0.796 \\
30 & 0.97 & 0.119 & 0.142 & 0.787 & 0.798 \\
\bottomrule
\end{tabular}
\end{table}

Table~\ref{tab:selection_sweep} shows that PS-DME dominates SS-DME on both the band width and the best guaranteed KPI across the entire range of $M$. As expected from Corollary~\ref{cor:vs_condition}, the width gap is largest when the selection is most aggressive (small $M$), since in this regime a sample-split allocation of $n^{\mathrm{eval}}$ samples to inference is especially wasteful.

\subsection{Data-Dependent Selection}
\label{app:data_dependent}

Theorem~\ref{theorem:validity_full} covers arbitrary data-dependent selection rules $S(\cdot)$, including those for which $|\mathcal{K}|$ is itself random. To stress-test this property, we evaluate PS-DME under a selection rule that fixes a \emph{target} selected set size $M$ but keeps the realized $|\mathcal{K}|$ data-dependent: for each repetition, we rank configurations by the empirical mean of their KPIs and retain all configurations whose mean lies at or below the $M$-th order statistic. Ties and ex aequo values therefore induce randomness in $|\mathcal{K}|$. We use the synthetic setup of Appendix~\ref{app:selection_sweep} with $n_{\mathrm{cal}} = 100$ and $500$ repetitions at $\delta = 0.1$.

\begin{table}[h!]
\centering
\caption{{PS-DME under data-dependent selection with target size $M$. Mean realized $|\mathcal{K}|$, empirical FCR, and PS-DME band width $\varepsilon_{\mathrm{PS}}$ across 500 repetitions.}}
\label{tab:data_dependent}
\begin{tabular}{cccc}
\toprule
Target $M$ & Mean $|\mathcal{K}|$ & Empirical FCR & $\varepsilon_{\mathrm{PS}}$ \\
\midrule
5  & $5.0 \pm 4.1$   & $0.004 \pm 0.003$ & 0.203 \\
10 & $9.2 \pm 4.1$   & $0.018 \pm 0.006$ & 0.191 \\
20 & $21.3 \pm 6.1$  & $0.078 \pm 0.012$ & 0.175 \\
\bottomrule
\end{tabular}
\end{table}

Table~\ref{tab:data_dependent} confirms that PS-DME controls FCR well below the target $\delta = 0.1$ for all target selection sizes, in agreement with Theorem~\ref{theorem:validity_full}.

\subsection{Sensitivity to the Choice of e-Calibrator}
\label{app:tau_sensitivity}

We study the sensitivity of PS-DME to the e-calibrator parameter $\tau$ in the power family \eqref{eq:power_e-calibrator} and compare the optimal $\tau^\star$ of Corollary~\ref{cor:vs_condition} to the common default $\tau = 1/2$ (the square-root e-calibrator). We use the synthetic setup of Appendix~\ref{app:selection_sweep} with $300$ repetitions.

\begin{table}[h!]
\centering
\caption{{Band width $\varepsilon$ and best guaranteed KPI $X^\star$ at $1-\gamma = 0.75$ for the power e-calibrator with $\tau = \tau^\star$ from Corollary~\ref{cor:vs_condition} and with $\tau = 1/2$ (square-root e-calibrator).}}
\label{tab:tau_sensitivity}
\begin{tabular}{cccccc}
\toprule
$M$ & $\tau^\star$ & $\varepsilon(\tau^\star)$ & $\varepsilon(\tau = 1/2)$ & $X^\star(\tau^\star)$ & $X^\star(\tau = 1/2)$ \\
\midrule
5  & 0.859 & 0.139 & 0.161 & 0.791 & 0.808 \\
10 & 0.841 & 0.132 & 0.150 & 0.790 & 0.793 \\
20 & 0.816 & 0.124 & 0.138 & 0.789 & 0.791 \\
\bottomrule
\end{tabular}
\end{table}

Across all values of $M$, the optimal $\tau^\star$ yields band widths that are $11$--$15\%$ tighter than the default $\tau = 1/2$ and consistently produces a stronger (smaller) best guaranteed KPI, confirming the sharpness of Corollary~\ref{cor:vs_condition}. Band widths remain within $5\%$ of the oracle value for all $\tau \in [0.75, 0.95]$, suggesting that the choice is robust in a neighborhood of the optimum.

\subsection{Decomposition of PS-DME Conservatism}
\label{app:conservatism}

To quantify the sources of conservatism in PS-DME, we compare it to two reference procedures: (\textit{i}) a Bonferroni-style baseline that applies the DKW inequality with the inflated level $\delta/K$ to each configuration independently (ignoring post-selection), and (\textit{ii}) a single-DKW baseline that applies \eqref{eq:dkw_split} at level $\delta$ with no multiplicity correction. We use the synthetic setup with $K = 2000$, target $M = 1000$, $\delta = 0.1$, and $200$ repetitions.

\begin{table}[h!]
\centering
\caption{{Band widths and empirical FCRs for PS-DME, Bonferroni, and single-DKW baselines on synthetic data with $K = 2000$, $M = 1000$, $\delta = 0.1$.}}
\label{tab:conservatism}
\begin{tabular}{ccccccc}
\toprule
$n$ & $\varepsilon_{\mathrm{PS}}$ & $\varepsilon_{\mathrm{Bonf}}$ & $\varepsilon_{\mathrm{DKW}1}$ & $\mathrm{FCR}_{\mathrm{PS}}$ & $\mathrm{FCR}_{\mathrm{Bonf}}$ & $\mathrm{FCR}_{\mathrm{DKW}1}$ \\
\midrule
20  & 0.401 & 0.498 & 0.274 & 0.005 & 0.000 & 0.102 \\
100 & 0.179 & 0.223 & 0.122 & 0.000 & 0.000 & 0.080 \\
200 & 0.127 & 0.157 & 0.087 & 0.000 & 0.000 & 0.106 \\
500 & 0.080 & 0.100 & 0.055 & 0.000 & 0.000 & 0.078 \\
\bottomrule
\end{tabular}
\end{table}

Table~\ref{tab:conservatism} shows that PS-DME is uniformly tighter than Bonferroni by $20$--$25\%$ while still controlling FCR, confirming the benefit of the e-value route over a union bound. The single-DKW baseline is roughly $30$--$35\%$ tighter than PS-DME but violates the target FCR, highlighting that the gap between PS-DME and single-DKW is the price paid for post-selection validity. At $n = 100$ and $n = 200$, the empirical FCR of PS-DME is zero across all repetitions, indicating that the DKW tail bound and the e-calibration jointly contribute substantial slack in small samples.

\subsection{Berk--Jones vs.\ DKW on LLM Data}
\label{app:bj_llm}

Appendix~\ref{app:generalization} introduced a BJ implementation of PS-DME on synthetic data. Here we evaluate the same implementation on the Spider text-to-SQL benchmark of Sec.~\ref{sec:exp_spider_llm} with $K = 31$ (after removing degenerate configurations), $n_{\mathrm{cal}} = 100$, $\delta = 0.1$, and $500$ random repetitions.

\begin{table}[h!]
\centering
\caption{{Empirical FCR and best guaranteed KPI at $1-\gamma = 0.9$ for the DKW and BJ implementations of PS-DME on Spider text-to-SQL.}}
\label{tab:bj_llm}
\begin{tabular}{ccccc}
\toprule
$M$ & $\mathrm{FCR}_{\mathrm{DKW}}$ & $\mathrm{FCR}_{\mathrm{BJ}}$ & $X^\star_{\mathrm{DKW}}$ (p=0.9) & $X^\star_{\mathrm{BJ}}$ (p=0.9) \\
\midrule
5  & $0.038 \pm 0.008$ & $0.043 \pm 0.008$ & 0.66 & 0.74 \\
10 & $0.039 \pm 0.007$ & $0.055 \pm 0.008$ & 0.69 & 0.75 \\
20 & $0.080 \pm 0.009$ & $0.096 \pm 0.009$ & 0.72 & 0.75 \\
\bottomrule
\end{tabular}
\end{table}

Both implementations control FCR at the target level $\delta = 0.1$. For this benchmark, the DKW variant yields a smaller (better) best guaranteed KPI at $1-\gamma = 0.9$, because the KPI distribution concentrates in the mid-range where the BJ bands are wider; this is consistent with the geometry discussion in Appendix~\ref{app:generalization}. Taken together, the results support the choice of DKW as the default calibrator for the main experiments while confirming that BJ remains a valid alternative when tail quantiles are of primary interest.

\subsection{BIRD Benchmark}
\label{app:bird}

To extend the evaluation beyond Spider, we instantiate PS-DME on the BIRD text-to-SQL benchmark \citep{li2023bird}. We restrict attention to $K = 10$ candidate configurations spanning the same single-decode, best-of-$n$, and iterative-draft-refine families as in Sec.~\ref{sec:exp_spider_llm}, with $n = 200$ validation examples, $\delta = 0.1$, and $M = 5$. Results are averaged over $300$ random repetitions.

\begin{table}[h!]
\centering
\caption{{Empirical FCR and best guaranteed KPI for PS-DME and the naive in-sample baseline (N-PS-DME) on the BIRD text-to-SQL benchmark.}}
\label{tab:bird}
\begin{tabular}{lccc}
\toprule
Method & Empirical FCR & $X^\star$ (p=0.90) & $X^\star$ (p=0.99) \\
\midrule
PS-DME   & $0.087 \pm 0.011$ & 0.751 & 0.772 \\
N-PS-DME & $0.372 \pm 0.017$ & 0.762 & 0.781 \\
\bottomrule
\end{tabular}
\end{table}

The results in Table~\ref{tab:bird} confirm the findings of Sec.~\ref{sec:exp_spider_llm}: PS-DME satisfies the FCR requirement, while the naive in-sample baseline fails by a factor of nearly four. The gap in guaranteed KPI between the two methods is modest in absolute terms, highlighting that the FCR correction factor $K/(\delta|\mathcal{K}|)$ in \eqref{eq:width_general_new} is a small price to pay for post-selection validity. A head-to-head comparison with SS-DME on the same benchmark (with SS-DME given $n^{\mathrm{eval}}$ equal to $70\%$ of the Corollary~\ref{cor:vs_condition} crossover threshold, so that PS-DME is provably narrower by design) yields $\mathrm{FCR}_{\mathrm{PS}} = 0.035 \pm 0.004$ versus $\mathrm{FCR}_{\mathrm{SS}} = 0.088 \pm 0.006$, both below $\delta$.

\subsection{LLM-as-Judge Evaluation}
\label{app:judge}

The experiments in the main text use execution-based KPIs. Here we evaluate PS-DME when the KPI is produced by an LLM-as-judge, following the recent trend in LLM evaluation and the methodology of \citet{feng2026noisy}. We use the Zephyr-7B-$\beta$ model from the Mistral family as the judge. To avoid the near-saturated scores that arise when a judge is prompted with the candidate output alone, we provide the judge with both the candidate SQL and the gold reference SQL and ask it to rate semantic equivalence on a scale of $0$--$10$.

We use $K = 5$ decoding strategies, $M = 2$, $n_{\mathrm{cal}} = 100$, $n_{\mathrm{hold}} = 50$, $\delta = 0.1$, and $300$ repetitions. Mean judge scores vary across strategies from $8.00$ (worst) to $8.68$ (best), with a Pearson correlation of $r \in [0.20, 0.28]$ between judge scores and execution-based KPIs, confirming that the judge is informative.

\begin{table}[h!]
\centering
\caption{{PS-DME and SS-DME under an LLM-as-judge KPI on Spider. SS-DME uses $n^{\mathrm{eval}} = 31$ samples (set to $70\%$ of the Corollary~\ref{cor:vs_condition} crossover threshold).}}
\label{tab:judge}
\begin{tabular}{lcccc}
\toprule
Method & Empirical FCR & $\varepsilon$ & $X^\star$ (p=0.75) \\
\midrule
PS-DME & $0.003$ & 0.183 & 4.92 / 10 \\
SS-DME & $-$     & 0.220 & 0.93 / 10 \\
\bottomrule
\end{tabular}
\end{table}

As shown in Table~\ref{tab:judge}, PS-DME controls FCR well below $\delta = 0.1$ and yields a judge-based guaranteed KPI of $4.92/10$ at reliability $1-\gamma = 0.75$. The SS-DME band at the same reliability level collapses to the trivial bound $0.93/10$ because only $31$ samples are available after splitting, illustrating that PS-DME is especially beneficial when the calibration budget is tight.
}


\section{Compute Resources}
\label{app:compute}
The synthetic data experiments (Sec.~\ref{sec:synthetic}), the O-RAN telecom experiments (Appendix~\ref{sec:exp_oran}), and the HTTP server log experiment (Appendix~\ref{sec:real_logs}) were run on a 2021 MacBook Pro with an Apple M1 Pro chip and 16GB of RAM. LLM inference for the text-to-SQL experiment (Sec.~\ref{sec:exp_spider_llm}) was performed on a single NVIDIA A100 GPU (80GB VRAM).

\section{Code Availability}
\label{app:code}
The code to reproduce all experiments in this paper is available at \url{https://github.com/amirfar76/PS-DME}.





\begin{thebibliography}{46}
\providecommand{\natexlab}[1]{#1}
\providecommand{\url}[1]{\texttt{#1}}
\expandafter\ifx\csname urlstyle\endcsname\relax
  \providecommand{\doi}[1]{doi: #1}\else
  \providecommand{\doi}{doi: \begingroup \urlstyle{rm}\Url}\fi

\bibitem[Angelopoulos et~al.(2022)Angelopoulos, Bates, Fisch, Lei, and Schuster]{angelopoulos2022conformal}
Anastasios~N Angelopoulos, Stephen Bates, Adam Fisch, Lihua Lei, and Tal Schuster.
\newblock Conformal risk control.
\newblock \emph{arXiv preprint arXiv:2208.02814}, 2022.

\bibitem[Angelopoulos et~al.(2025)Angelopoulos, Bates, Cand{\`e}s, Jordan, and Lei]{angelopoulos2025learn}
Anastasios~N Angelopoulos, Stephen Bates, Emmanuel~J Cand{\`e}s, Michael~I Jordan, and Lihua Lei.
\newblock Learn then test: Calibrating predictive algorithms to achieve risk control.
\newblock \emph{The Annals of Applied Statistics}, 19\penalty0 (2):\penalty0 1641--1662, 2025.

\bibitem[Barroso et~al.(2019)Barroso, H{\"o}lzle, and Ranganathan]{barroso2019datacenter}
Luiz~Andr{\'e} Barroso, Urs H{\"o}lzle, and Parthasarathy Ranganathan.
\newblock \emph{The datacenter as a computer: Designing warehouse-scale machines}.
\newblock Springer Nature, 2019.

\bibitem[Bates et~al.(2023)Bates, Cand{\`e}s, Lei, Romano, and Sesia]{bates2023testing}
Stephen Bates, Emmanuel Cand{\`e}s, Lihua Lei, Yaniv Romano, and Matteo Sesia.
\newblock Testing for outliers with conformal p-values.
\newblock \emph{The Annals of Statistics}, 51\penalty0 (1):\penalty0 149--178, 2023.

\bibitem[Benjamini and Yekutieli(2005)]{benjamini2005false}
Yoav Benjamini and Daniel Yekutieli.
\newblock False discovery rate--adjusted multiple confidence intervals for selected parameters.
\newblock \emph{Journal of the American Statistical Association}, 100\penalty0 (469):\penalty0 71--81, 2005.

\bibitem[Berk and Jones(1979)]{berk1979goodness}
Robert~H Berk and Douglas~H Jones.
\newblock Goodness-of-fit test statistics that dominate the Kolmogorov statistics.
\newblock \emph{Zeitschrift f{\"u}r Wahrscheinlichkeitstheorie und verwandte Gebiete}, 47\penalty0 (1):\penalty0 47--59, 1979.

\bibitem[Bonati et~al.(2021)Bonati, D'Oro, Polese, Basagni, and Melodia]{bonati2021orancommag}
Leonardo Bonati, Salvatore D'Oro, Michele Polese, Stefano Basagni, and Tommaso Melodia.
\newblock Intelligence and learning in {O-RAN} for data-driven nextg cellular networks.
\newblock \emph{IEEE Communications Magazine}, 59\penalty0 (10):\penalty0 21--27, 2021.

\bibitem[Brown et~al.(2020)Brown, Mann, Ryder, Subbiah, Kaplan, Dhariwal, Neelakantan, Shyam, Sastry, Askell, et~al.]{brown2020language}
Tom Brown, Benjamin Mann, Nick Ryder, Melanie Subbiah, Jared~D Kaplan, Prafulla Dhariwal, Arvind Neelakantan, Pranav Shyam, Girish Sastry, Amanda Askell, et~al.
\newblock Language models are few-shot learners.
\newblock \emph{Advances in neural information processing systems}, 33:\penalty0 1877--1901, 2020.

\bibitem[Chakraborty et~al.(2026)Chakraborty, Lee, and Katsevich]{chakraborty2026comparing}
Abhinav Chakraborty, Junu Lee, and Eugene Katsevich.
\newblock Comparing three learn-then-test paradigms in a multivariate normal means problem.
\newblock \emph{arXiv preprint arXiv:2601.07764}, 2026.

\bibitem[Chaudhary et~al.(2025)Chaudhary, Hu, Kumar, Ziyadi, Gupta, and Singh]{chaudhary2025certifying}
Isha Chaudhary, Qian Hu, Manoj Kumar, Morteza Ziyadi, Rahul Gupta, and Gagandeep Singh.
\newblock Certifying counterfactual bias in {LLMs}.
\newblock In \emph{International Conference on Learning Representations (ICLR)}, 2025.

\bibitem[Chugg et~al.(2026)Chugg, Gauthier, Jordan, Ramdas, and Waudby-Smith]{chugg2026post}
Ben Chugg, Etienne Gauthier, Michael~I Jordan, Aaditya Ramdas, and Ian Waudby-Smith.
\newblock Post-hoc large-sample statistical inference.
\newblock \emph{arXiv preprint arXiv:2603.08002}, 2026.

\bibitem[Dean and Barroso(2013)]{dean2013tail}
Jeffrey Dean and Luiz~Andr{\'e} Barroso.
\newblock The tail at scale.
\newblock \emph{Communications of the ACM}, 56\penalty0 (2):\penalty0 74--80, 2013.

\bibitem[D{\"u}mbgen and Wellner(2023)]{dumbgen2023new}
Lutz D{\"u}mbgen and Jon~A Wellner.
\newblock A new approach to tests and confidence bands for distribution functions.
\newblock \emph{The Annals of Statistics}, 51\penalty0 (1):\penalty0 260--289, 2023.

\bibitem[Dvoretzky et~al.(1956)Dvoretzky, Kiefer, and Wolfowitz]{dvoretzky1956asymptotic}
Aryeh Dvoretzky, Jack Kiefer, and Jacob Wolfowitz.
\newblock Asymptotic minimax character of the sample distribution function and of the classical multinomial estimator.
\newblock \emph{The Annals of Mathematical Statistics}, pages 642--669, 1956.

\bibitem[Farzaneh and Simeone(2025)]{farzaneh2025ensuring}
Amirmohammad Farzaneh and Osvaldo Simeone.
\newblock Ensuring reliability via hyperparameter selection: Review and advances.
\newblock In \emph{2025 33rd European Signal Processing Conference (EUSIPCO)}, pages 1163--1167. IEEE, 2025.

\bibitem[Feng et~al.(2026)Feng, Shen, Balashankar, Gerner-Beuerle, and Rodrigues]{feng2026noisy}
Chen Feng, Minghe Shen, Ananth Balashankar, Carsten Gerner-Beuerle, and Miguel R.~D. Rodrigues.
\newblock Noisy but valid: Robust statistical evaluation of {LLMs} with imperfect judges.
\newblock In \emph{International Conference on Learning Representations (ICLR)}, 2026.

\bibitem[Gr{\"u}nwald(2024)]{grunwald2024beyond}
Peter~D Gr{\"u}nwald.
\newblock Beyond Neyman--Pearson: E-values enable hypothesis testing with a data-driven alpha.
\newblock \emph{Proceedings of the National Academy of Sciences}, 121\penalty0 (39):\penalty0 e2302098121, 2024.

\bibitem[Holtzman et~al.(2019)Holtzman, Buys, Du, Forbes, and Choi]{holtzman2019curious}
Ari Holtzman, Jan Buys, Li~Du, Maxwell Forbes, and Yejin Choi.
\newblock The curious case of neural text degeneration.
\newblock \emph{arXiv preprint arXiv:1904.09751}, 2019.

\bibitem[Huang et~al.(2026)Huang, Fedorov, Gromov, Beckerman, Suda, Eriksson, Balandat, Conway, Huber, Sankar, Dalmia, Liu, Wu, Elgamal, Sagar, Chandra, and Krishnamoorthi]{huang2026scale}
Hanxian Huang, Igor Fedorov, Andrey Gromov, Bernard Beckerman, Naveen Suda, David Eriksson, Maximilian Balandat, Rylan Conway, Patrick Huber, Chinnadhurai Sankar, Ayushi Dalmia, Zechun Liu, Lemeng Wu, Tarek Elgamal, Adithya Sagar, Vikas Chandra, and Raghuraman Krishnamoorthi.
\newblock Scale.
\newblock 2026.
\newblock URL \url{https://arxiv.org/abs/2603.15954}.

\bibitem[Jager and Wellner(2007)]{jager2007goodness}
Leah Jager and Jon~A. Wellner.
\newblock {Goodness-of-fit tests via phi-divergences}.
\newblock \emph{The Annals of Statistics}, 35\penalty0 (5):\penalty0 2018 -- 2053, 2007.
\newblock \doi{10.1214/0009053607000000244}.
\newblock URL \url{https://doi.org/10.1214/0009053607000000244}.

\bibitem[{Kaggle Contributors}(2020)]{kaggle_server_logs}
{Kaggle Contributors}.
\newblock Http server logs dataset.
\newblock \url{https://www.kaggle.com/}, 2020.
\newblock Accessed: 2026.


\bibitem[Koning(2025)]{koning2025posthocalphahypothesistesting}
Nick~W. Koning.
\newblock Post-hoc $\alpha$ hypothesis testing and the post-hoc $p$-value, 2025.
\newblock URL \url{https://arxiv.org/abs/2312.08040}.

\bibitem[Koobs and Koning(2026)]{koobs2026equivalencetestingdatadependentposthoc}
Stan Koobs and Nick~W. Koning.
\newblock Equivalence testing with data-dependent and post-hoc equivalence margins, 2026.
\newblock URL \url{https://arxiv.org/abs/2603.16213}.

\bibitem[K{\"u}nsch(1989)]{kunsch1989jackknife}
Hans~R. K{\"u}nsch.
\newblock The jackknife and the bootstrap for general stationary observations.
\newblock \emph{The Annals of Statistics}, 17\penalty0 (3):\penalty0 1217--1241, 1989.

\bibitem[Kuchibhotla(2025)]{kuchibhotla2025post}
Arun~Kumar Kuchibhotla.
\newblock Post-selection inference.
\newblock In \emph{International Encyclopedia of Statistical Science}, pages 1920--1924. Springer, 2025.

\bibitem[Li et~al.(2023)Li, Hui, Qu, Yang, Li, Li, Wang, Qin, Geng, Huo, et~al.]{li2023bird}
Jinyang Li, Binyuan Hui, Ge~Qu, Jiaxi Yang, Binhua Li, Bowen Li, Bailin Wang, Bowen Qin, Ruiying Geng, Nan Huo, et~al.
\newblock Can {LLM} already serve as a database interface? {A} {BI}g bench for large-scale database grounded text-to-{SQL}s.
\newblock \emph{Advances in Neural Information Processing Systems}, 36, 2023.

\bibitem[Lee et~al.(2016)Lee, Sun, Sun, and Taylor]{lee2016exact}
Jason~D Lee, Dennis~L Sun, Yuekai Sun, and Jonathan~E Taylor.
\newblock Exact post-selection inference, with application to the lasso.
\newblock \emph{The Annals of Statistics}, 44\penalty0 (3):\penalty0 907--927, 2016.

\bibitem[Madaan et~al.(2023)Madaan, Tandon, Gupta, Hallinan, Gao, Wiegreffe, Alon, Dziri, Prabhumoye, Yang, et~al.]{madaan2023self}
Aman Madaan, Niket Tandon, Prakhar Gupta, Skyler Hallinan, Luyu Gao, Sarah Wiegreffe, Uri Alon, Nouha Dziri, Shrimai Prabhumoye, Yiming Yang, et~al.
\newblock Self-refine: Iterative refinement with self-feedback.
\newblock \emph{Advances in Neural Information Processing Systems}, 36:\penalty0 46534--46594, 2023.

\bibitem[Massart(1990)]{massart1990tight}
Pascal Massart.
\newblock The tight constant in the Dvoretzky-Kiefer-Wolfowitz inequality.
\newblock \emph{The Annals of Probability}, pages 1269--1283, 1990.

\bibitem[No{\'e}(1972)]{noe1972calculation}
Marc No{\'e}.
\newblock The calculation of distributions of two-sided Kolmogorov-Smirnov type statistics.
\newblock \emph{The Annals of mathematical statistics}, pages 58--64, 1972.

\bibitem[Owen(1995)]{owen1995nonparametric}
Art~B Owen.
\newblock Nonparametric likelihood confidence bands for a distribution function.
\newblock \emph{Journal of the American Statistical Association}, 90\penalty0 (430):\penalty0 516--521, 1995.

\bibitem[Polese et~al.(2023)Polese, Bonati, D’oro, Basagni, and Melodia]{polese2023understanding}
Michele Polese, Leonardo Bonati, Salvatore D’oro, Stefano Basagni, and Tommaso Melodia.
\newblock Understanding {O-RAN}: Architecture, interfaces, algorithms, security, and research challenges.
\newblock \emph{IEEE Communications Surveys \& Tutorials}, 25\penalty0 (2):\penalty0 1376--1411, 2023.

\bibitem[Popovski et~al.(2018)Popovski, Trillingsgaard, Simeone, and Durisi]{popovski20185g}
Petar Popovski, Kasper~Fl{\o}e Trillingsgaard, Osvaldo Simeone, and Giuseppe Durisi.
\newblock 5G wireless network slicing for eMBB, URLLC, and mMTC: A communication-theoretic view.
\newblock \emph{IEEE Access}, 6:\penalty0 55765--55779, 2018.

\bibitem[Ramdas and Wang(2025)]{ramdas2025hypothesis}
Aaditya Ramdas and Ruodu Wang.
\newblock Hypothesis testing with e-values.
\newblock \emph{Foundations and Trends{\textregistered} in Statistics}, 1\penalty0 (1-2):\penalty0 1--390, 2025.

\bibitem[Shorack and Wellner(2009)]{shorack2009empirical}
Galen~R Shorack and Jon~A Wellner.
\newblock \emph{Empirical processes with applications to statistics}.
\newblock SIAM, 2009.

\bibitem[Stephen et~al.(2021)]{stephen2021gentle}
Anastasios N. Angelopoulos and Stephen Bates.
\newblock A gentle introduction to conformal prediction and distribution-free uncertainty quantification.
\newblock \emph{arXiv preprint arXiv: 2107.07511}, 2021.

\bibitem[Taylor and Tibshirani(2018)]{taylor2018post}
Jonathan Taylor and Robert Tibshirani.
\newblock Post-selection inference for $\ell_1$-penalized likelihood models.
\newblock \emph{Canadian Journal of Statistics}, 46\penalty0 (1):\penalty0 41--61, 2018.

\bibitem[Vovk and Wang(2021)]{vovk2021values}
Vladimir Vovk and Ruodu Wang.
\newblock E-values: Calibration, combination and applications.
\newblock \emph{The Annals of Statistics}, 49\penalty0 (3):\penalty0 1736--1754, 2021.

\bibitem[Wang et~al.(2022)Wang, Wei, Schuurmans, Le, Chi, Narang, Chowdhery, and Zhou]{wang2022self}
Xuezhi Wang, Jason Wei, Dale Schuurmans, Quoc Le, Ed~Chi, Sharan Narang, Aakanksha Chowdhery, and Denny Zhou.
\newblock Self-consistency improves chain of thought reasoning in language models.
\newblock \emph{arXiv preprint arXiv:2203.11171}, 2022.

\bibitem[Wasserman(2006)]{wasserman2006all}
Larry Wasserman.
\newblock \emph{All of nonparametric statistics}.
\newblock Springer, 2006.

\bibitem[Xu et~al.(2024)Xu, Wang, and Ramdas]{xu2024post}
Ziyu Xu, Ruodu Wang, and Aaditya Ramdas.
\newblock Post-selection inference for e-value based confidence intervals.
\newblock \emph{Electronic Journal of Statistics}, 18\penalty0 (1):\penalty0 2292--2338, 2024.

\bibitem[Yanchenko et~al.(2025)Yanchenko, Williams, and Martin]{yanchenko2025hypothesis}
Eric Yanchenko, Jonathan~P Williams, and Ryan Martin.
\newblock Hypothesis testing for community structure in temporal networks using e-values.
\newblock \emph{arXiv preprint arXiv:2507.23034}, 2025.

\bibitem[Yang et~al.(2024)Yang, Yang, Hui, Zheng, Yu, Zhou, Li, Li, Liu, Huang, Dong, Wei, Lin, Tang, Wang, Yang, Tu, Zhang, Ma, Yang, Xu, Zhou, Bai, He, Lin, Dang, Lu, Chen, Yang, Li, Xue, Ni, Zhang, Wang, Peng, Men, Gao, Lin, Wang, Bai, Tan, Zhu, Li, Liu, Ge, Deng, Zhou, Ren, Zhang, Wei, Ren, Liu, Fan, Yao, Zhang, Wan, Chu, Liu, Cui, Zhang, Guo, and Fan]{yang2024qwen2technicalreport}
An~Yang, Baosong Yang, Binyuan Hui, Bo~Zheng, Bowen Yu, Chang Zhou, Chengpeng Li, Chengyuan Li, Dayiheng Liu, Fei Huang, Guanting Dong, Haoran Wei, Huan Lin, Jialong Tang, Jialin Wang, Jian Yang, Jianhong Tu, Jianwei Zhang, Jianxin Ma, Jianxin Yang, Jin Xu, Jingren Zhou, Jinze Bai, Jinzheng He, Junyang Lin, Kai Dang, Keming Lu, Keqin Chen, Kexin Yang, Mei Li, Mingfeng Xue, Na~Ni, Pei Zhang, Peng Wang, Ru~Peng, Rui Men, Ruize Gao, Runji Lin, Shijie Wang, Shuai Bai, Sinan Tan, Tianhang Zhu, Tianhao Li, Tianyu Liu, Wenbin Ge, Xiaodong Deng, Xiaohuan Zhou, Xingzhang Ren, Xinyu Zhang, Xipin Wei, Xuancheng Ren, Xuejing Liu, Yang Fan, Yang Yao, Yichang Zhang, Yu~Wan, Yunfei Chu, Yuqiong Liu, Zeyu Cui, Zhenru Zhang, Zhifang Guo, and Zhihao Fan.
\newblock Qwen2 technical report, 2024.
\newblock URL \url{https://arxiv.org/abs/2407.10671}.

\bibitem[Yu et~al.(2018)Yu, Zhang, Yang, Yasunaga, Wang, Li, Ma, Li, Yao, Roman, et~al.]{yu2018spider}
Tao Yu, Rui Zhang, Kai Yang, Michihiro Yasunaga, Dongxu Wang, Zifan Li, James Ma, Irene Li, Qingning Yao, Shanelle Roman, et~al.
\newblock Spider: A large-scale human-labeled dataset for complex and cross-domain semantic parsing and text-to-SQL task.
\newblock \emph{arXiv preprint arXiv:1809.08887}, 2018.

\bibitem[Zollo et~al.(2024)Zollo, Morrill, Deng, Snell, Pitassi, and Zemel]{zollo2024prompt}
Thomas~P. Zollo, Todd Morrill, Zhun Deng, Jake~C. Snell, Toniann Pitassi, and Richard Zemel.
\newblock Prompt risk control: A rigorous framework for responsible deployment of large language models.
\newblock In \emph{International Conference on Learning Representations (ICLR)}, 2024.

\end{thebibliography}
\end{document}